\documentclass[11pt]{article}

\usepackage[margin=1in]{geometry}
\usepackage{amsmath,amssymb}
\usepackage{mathtools}
\usepackage{booktabs,multirow,longtable}
\usepackage{graphicx}
\usepackage{xcolor}
\usepackage{url}
\usepackage[algo2e,ruled]{algorithm2e}
\usepackage{natbib}
\bibpunct{(}{)}{;}{a}{,}{,}
\usepackage{jmlrutils}
\usepackage[colorlinks=true,citecolor=blue,linkcolor=blue,urlcolor=blue]{hyperref}

\newcommand{\Ilow}{I_{\mathrm{low}}}
\newcommand{\Imid}{I_{\mathrm{mid}}}
\newcommand{\Ihigh}{I_{\mathrm{high}}}
\newcommand{\calR}{\mathcal{R}}
\newcommand{\calC}{\mathcal{C}}

\newcommand{\hatlam}{\hat{\lambda}}

\newcommand*{\propositionrefname}{Proposition}
\newcommand*{\propositionsrefname}{Propositions}
\newcommand*{\propositionref}[1]{%
  \objectref{#1}{\propositionrefname}{\propositionsrefname}{}{}}

\newenvironment{keywords}
{\bgroup\leftskip 20pt\rightskip 20pt \small\noindent{\bfseries
Keywords:} \ignorespaces}%
{\par\egroup\vskip 0.25ex}

\title{Scaling-Score Conformal Prediction for Multi-Target Regression}

\author{%
  Sylvain Rousseau \and Soundouss Messoudi \\[4pt]
  \normalsize Université de Technologie de Compiègne, CNRS, Heudiasyc, Compiègne, France \\
  \normalsize\texttt{sylvain.rousseau@hds.utc.fr}, \texttt{soundouss.messoudi@hds.utc.fr}
}
\date{}

\begin{document}

\maketitle

%% ==========================================================================
\begin{abstract}
Multi-target regression requires a model to simultaneously predict several
related outputs.  Conformal prediction provides distribution-free,
finite-sample marginal coverage guarantees, but extending these to joint
multi-dimensional regions in a model-agnostic, sample-efficient manner
remains challenging: max-aggregation ignores scale differences, copula-based
methods are only asymptotically valid, rectangular methods typically split
the calibration set, and quantile- or density-based methods require training
a specialised model beyond a plain point predictor.  We propose the
\emph{scaling-score} conformal method, which is model-agnostic (requires
only component-wise absolute residuals), uses a single calibration set, and
yields four nested output types: an \emph{outer rectangle} (SCO) with valid joint
coverage, the exact set $\calR_\alpha$, a \emph{staircase} ($\mathrm{SC}^2$) over-approximation
of $\calR_\alpha$, and an \emph{inner rectangle} (SCI).
A single hyperparameter $\gamma \in (0,1)$ controls the base-rectangle
quantile level independently of $\alpha$.  We prove downward-closedness and
a rectangular sandwich bound and derive a closed-form outer rectangle.
Experiments on 29 real-world datasets confirm valid joint coverage; $\mathrm{SC}^2$ with $\gamma=1-\alpha$ consistently achieves competitive volume
relative to baselines, with the advantage growing with output dimension $d$.
\end{abstract}


%% ==========================================================================
\section{Introduction}
\label{sec:intro}
%% ==========================================================================

Modern applications increasingly require the simultaneous prediction of several
correlated response variables.  Such settings arise whenever the response is
intrinsically multivariate: energy systems, environmental monitoring, finance,
and supply chains all produce outputs whose components must be predicted jointly.
In all these settings, downstream decisions depend on the \emph{joint} distribution of the
outputs, making a marginal error bar for each target individually insufficient.
The relevant object is a prediction \emph{region} $\mathcal{C}(x) \subseteq \mathbb{R}^d$
that simultaneously covers all $d$ components with a specified probability.

Conformal prediction \citep{vovk_algorithmic_2005} provides prediction regions
with a finite-sample, distribution-free coverage guarantee: under
exchangeability of the calibration and test data,
$\mathbb{P}[Y_{n+1} \in \mathcal{C}(X_{n+1})] \geq 1-\alpha$ for any miscoverage level
$\alpha \in (0,1)$, irrespective of the data-generating distribution.  The method
is model-agnostic: it wraps any point predictor, and the guarantee requires no
distributional assumption beyond exchangeability.  Extending these guarantees to
$d$-dimensional outputs in a \emph{simultaneously valid, model-agnostic, and
sample-efficient} way is non-trivial, and current methods each suffer from at
least one key weakness.

The simplest multi-target approach, max aggregation, takes the maximum of
per-target absolute residuals; it yields valid coverage but produces
axis-aligned hypercubes that are over-conservative when targets have
heterogeneous scales. Copula-based methods \citep{messoudi_copulabased_2021}
model the joint dependence structure through a copula and directly construct a
prediction region; coverage guarantees are at best asymptotic. CHR
\citep{sampson_conformal_2024} achieves valid rectangular regions with
per-target width adaptation, but requires splitting the calibration set in two,
reducing the effective sample size for the conformal quantile. Transductively
Standardized Conformal Prediction \citep{fan_interpretable_2025} is valid and
model-agnostic with a single calibration set, but standardizes by the
empirical mean and standard deviation of calibration residuals, which
implicitly assumes these moments are well-defined and well-estimated;
normalizing by calibration quantiles instead avoids this assumption, since a
quantile depends only on the relative order of the observations rather than
on every calibration point. Methods based on conformalized quantile
regression \citep{romano_conformalized_2019}, conditional density estimation
\citep{izbicki_cdsplit_2020}, or normalizing flows require a trained model
beyond a plain point predictor.

We introduce a \emph{scaling score} that, for a given residual vector, measures
the minimum factor by which a per-target \emph{base vector} must be scaled to
contain it.  The base vector is derived from per-target calibration quantiles at
a tunable level, encoding the relative scale of each output coordinate.  The
base vector is defined transductively: it depends on the test residual through a
clamping formula, so that all calibration and test scores share the same base
and exchangeability is preserved.  From a single calibration set, our method simultaneously produces four
\emph{nested} output types (\figureref{fig:regions}): an \emph{outer rectangle} (SCO), a
single axis-aligned hyperrectangle that is valid but conservative; the
\emph{exact prediction set} $\calR_\alpha$, which is the tightest valid
prediction set but defined implicitly via a cell decomposition; a
\emph{staircase region} ($\mathrm{SC}^2$), a union of at most $2^d$ axis-aligned hyperrectangles
that is a valid over-approximation of $\calR_\alpha$, computable at calibration
time; and an \emph{inner rectangle} (SCI), the most compact
of the four, tighter even than $\calR_\alpha$, but without a formal coverage
guarantee.

Our contributions are threefold.  First, we propose a model-agnostic,
single-calibration-set conformal score for multi-target regression that
simultaneously produces four nested output types, three of which carry valid joint coverage guarantees.  Second, we derive an
exact cell-based analytical characterisation of the prediction region,
rectangular sandwich bounds with closed-form expressions, and a proof of
downward-closedness.  Third, experiments on 29 real-world datasets confirm
valid joint coverage and competitive efficiency across a range of target
dimensions.

\sectionref{sec:background} reviews conformal prediction and existing multi-target
methods.  \sectionref{sec:setup} establishes notation and problem setup.
\sectionref{sec:method} develops the scaling-score method.  \sectionref{sec:experiments}
reports experiments on 29 real-world datasets.  \sectionref{sec:conclusion}
concludes.  Proofs are collected in \appendixref{app:proofs}; dataset details appear
in \appendixref{app:datasets}; full experimental figures appear in \appendixref{app:full-results}.

%% ==========================================================================
\section{Background and Related Work}
\label{sec:background}
%% ==========================================================================

\subsection{Conformal prediction for regression}
\label{sec:background-cp}

Let $(X_1,Y_1),\ldots,(X_n,Y_n),(X_{n+1},Y_{n+1})$ be exchangeable random
variables taking values in $\mathcal{X} \times \mathbb{R}^d$.  Given a
pre-trained predictor $\hat{f}: \mathcal{X} \to \mathbb{R}^d$, a
\emph{nonconformity score} $s(x,y)$ measures how poorly $\hat{f}(x)$
predicts~$y$. In split (inductive) conformal prediction
\citep{papadopoulos_inductive_2002,vovk_algorithmic_2005}, the predictor is fitted on
training data and the score evaluated on a held-out calibration set
$\mathcal{D}_{\mathrm{cal}} = \{(X_i,Y_i)\}_{i=1}^n$.  Setting $s_i = s(X_i,Y_i)$,
the conformal quantile is
\[
  \hat{q}_{1-\alpha} = \text{$\lceil(1-\alpha)(n+1)\rceil$-th smallest value in }
    \{s_1,\ldots,s_n,+\infty\},
\]
and the prediction region $\mathcal{C}(x) = \{y : s(x,y) \leq \hat{q}_{1-\alpha}\}$
satisfies $\mathbb{P}[Y_{n+1} \in \mathcal{C}(X_{n+1})] \geq 1-\alpha$ by a simple
rank argument \citep{lei_distributionfree_2018}.

Standard nonconformity scores for univariate regression include the absolute
residual $|y-\hat{f}(x)|$ and the conformalized quantile regression score
\citep{romano_conformalized_2019}.  Normalized scores
\citep{papadopoulos_inductive_2002,lei_distributionfree_2018} divide the residual by a
local scale estimate to improve conditional coverage.

\subsection{Conformal prediction for multi-target regression}
\label{sec:background-joint}

With $d$-dimensional outputs, the joint coverage event $\{Y_{n+1} \in
\mathcal{C}(X_{n+1})\}$ requires the entire response vector to lie within the
prediction region $\mathcal{C}(X_{n+1}) \subseteq \mathbb{R}^d$.  Achieving at
least $1-\alpha$ joint coverage while keeping the region small is harder than
marginal coverage of each component individually.

The simplest approach uses the Max ($\ell^\infty$) score
$s(x,y) = \max_{k \in [d]} |y^{(k)} - \hat{f}^{(k)}(x)|$,
which is valid by the standard conformal argument and produces a hypercube
of radius $\hat{q}_{1-\alpha}$ centered at $\hat{f}(x)$.  However, the
hypercube is driven by the coordinate with the largest residuals; when targets
have heterogeneous scales, the resulting region is over-conservative for the
other targets.  The Bonferroni correction applied coordinate-wise achieves
$1 - d\alpha'$ joint coverage at level $\alpha' = \alpha/d$ per dimension, but
this is even more conservative.

Copula-based methods \citep{messoudi_copulabased_2021} model the joint
dependence structure of the residuals directly.  Parametric copulas assume a
specific family for the joint distribution of errors; the empirical copula
instead uses the rank transform, mapping each error coordinate $e_i^{(k)}$
through its marginal ECDF $\hat{F}_k$ to place all targets on a common $[0,1]$
scale.  Both approaches produce axis-aligned rectangular regions.  However,
neither defines a nonconformity score in the split-conformal sense, so
finite-sample coverage is not guaranteed.  The empirical copula achieves only
asymptotic validity as $n \to \infty$; parametric copulas additionally require
that the residuals follow the assumed copula distribution.

CHR \citep{sampson_conformal_2024} (Conformal Multi-Target Hyperrectangles) is a valid split-conformal method.
It uses a two-stage calibration: a first portion of the calibration set
estimates per-target prediction intervals (capturing scale differences across
targets), and a second portion computes a scaled nonconformity score on top of
the estimated intervals.  The method achieves finite-sample joint coverage and
asymptotic balance (equal marginal coverage across targets).  The main
limitation is that the required calibration split reduces the effective sample
size available for the conformal quantile, which can inflate the prediction
region when $n$ is moderate or $d$ is large.

TSCP and TSCP-R \citep{fan_interpretable_2025} (Transductively Standardized Conformal Prediction) standardize the residual
vector by the empirical mean and standard deviation of calibration residuals in a
transductive fashion, preserving exchangeability.  Both are valid, model-agnostic,
and use a single calibration set.  The TSCP-R variant restricts to a
rectangular output region.  These methods are the closest antecedent to our work; the key
difference is that TSCP standardizes by location and scale (mean $\pm$ std),
while our method normalizes purely by calibration quantiles,
giving more direct control over the coverage level of the base rectangle.

Density-based conformal methods \citep{izbicki_cdsplit_2020} can produce non-rectangular
regions that capture correlation geometry, and may therefore achieve smaller
volume than rectangular methods. However, they require learning a conditional
density estimator in addition to a point predictor, which adds modeling
complexity. Non-rectangular regions are also harder to interpret and communicate
than simple axis-aligned boxes. \citet{dheur_multioutput_2025} provide a
unified comparative study of several such methods.

%% ==========================================================================
\section{Problem Setup and Notation}
\label{sec:setup}
%% ==========================================================================

Let $\mathcal{X}$ be a feature space and $Y \in \mathbb{R}^d$ the response,
and let $\hat{f}: \mathcal{X} \to \mathbb{R}^d$ be a pre-trained point predictor.
We observe a calibration set $\mathcal{D}_{\mathrm{cal}} = \{(X_i, Y_i)\}_{i=1}^n$
and assume that $\{(X_i, Y_i)\}_{i=1}^{n+1}$ are exchangeable,
the sole assumption underlying all validity guarantees.

Define the component-wise absolute calibration errors
\[
  e_i = |Y_i - \hat{f}(X_i)| \in \mathbb{R}_+^d, \quad i = 1,\ldots,n.
\]
For a candidate test response $y \in \mathbb{R}^d$ define
$e_{n+1} = e_{n+1}(y) = |y - \hat{f}(X_{n+1})|$.
The goal is to construct a prediction region
$\mathcal{C}(X_{n+1}) \subseteq \mathbb{R}^d$ such that
$\mathbb{P}[Y_{n+1} \in \mathcal{C}(X_{n+1})] \geq 1-\alpha$.

Write $[d] = \{1,\ldots,d\}$.  For a vector $a \in \mathbb{R}^d$, we write
$a^{(k)}$ for its $k$-th coordinate; the same convention applies to $b$,
$e$, $Y$, and all other vectors throughout.  For $a, b \in \mathbb{R}^d$, write
$a \preceq b$ for the component-wise partial order: $a^{(k)} \leq b^{(k)}$
for all $k \in [d]$.  Let $e^{(k)}_{(j)}$ denote the $j$-th order statistic
of $\{e_1^{(k)},\ldots,e_n^{(k)}\}$.

%% ==========================================================================
\section{The Scaling-Score Method}
\label{sec:method}
%% ==========================================================================

Our method constructs a rectangular prediction region from a per-target
\emph{base rectangle} whose side lengths encode the relative scale of each
output.  The base rectangle is then inflated by the smallest scaling factor
(\figureref{fig:scoring}) that covers a $1-\alpha$ fraction of the calibration
points.  The main subtlety is that the base rectangle depends on the test error
through a clamping formula, and handling this dependency correctly while
preserving exchangeability is the central challenge of the construction.
We describe the construction in four steps: the scaling score
(\sectionref{sec:method-score}), the lower and upper base rectangles
(\sectionref{sec:method-base}), the cell decomposition
(\sectionref{sec:method-cells}), and the prediction region characterisation
and its properties (\sectionref{sec:method-region}).

\subsection{Scaling score}
\label{sec:method-score}

Given a base vector $b \in \mathbb{R}_{++}^d$, the scaling score measures the
minimum factor by which $b$ must be scaled to contain a residual $e$
(\figureref{fig:scoring}).

\begin{definition}[Scaling score]\label{def:scaling-score}
For a base vector $b \in \mathbb{R}_{++}^d$ and any $e \in \mathbb{R}_+^d$,
the \emph{scaling score} is the minimum dilation factor needed to fit $e$
inside the base rectangle $\prod_{k=1}^d [0,b^{(k)}]$:
\begin{equation}\label{eq:sigma}
  \sigma(e;\,b) = \max_{k \in [d]} \frac{e^{(k)}}{b^{(k)}}.
\end{equation}
\end{definition}

Geometrically, bounding $\sigma(e;\,b)$ by a threshold $\lambda$ is equivalent
to requiring that $e$ lies in the axis-aligned rectangle obtained by scaling
$b$ coordinate-wise by $\lambda$:
\begin{equation}\label{eq:containment}
  \sigma(e;\,b) \leq \lambda \;\Longleftrightarrow\; e \preceq \lambda b
  \;\Longleftrightarrow\; e \in \prod_{k=1}^d \bigl[0,\,\lambda\,b^{(k)}\bigr].
\end{equation}

\subsection{Lower and upper base rectangles}
\label{sec:method-base}

Fix a level $\gamma \in (0,1)$ and set $m' = \lceil \gamma(n+1)\rceil$.
Define the lower and upper base rectangles via their per-coordinate bounds
\begin{equation}\label{eq:LU}
  L_k = e_{(m'-1)}^{(k)}, \qquad U_k = e_{(m')}^{(k)}, \quad k \in [d].
\end{equation}
We require $2 \le m' \le n$ so that $L_k = e^{(k)}_{(m'-1)}$ is a
well-defined order statistic ($m' \ge 2$) and $U_k = e^{(k)}_{(m')}$ is
finite ($m' \le n$); in practice we clip $m' =
\operatorname{clip}(\lceil\gamma(n+1)\rceil, 2, n)$.
By construction $0 \leq L_k \leq U_k$ for all $k$.

The base rectangle is constructed transductively: its side lengths (each in $[L_k, U_k]$) depend on the
test residual $e_{n+1}$ through the rank it would occupy in the augmented
calibration set, keeping all scores exchangeable.

\begin{definition}[Transductive base vector]\label{def:base-vector}
The \emph{base vector} $b_{n+1} \in \mathbb{R}_{++}^d$ is the $m'$-th order
statistic of the pooled set in each coordinate:
\[
  b^{(k)}_{n+1} = \text{$m'$-th order statistic of }
    \{e_i^{(k)}\}_{i=1}^{n+1}, \quad k \in [d].
\]
The associated \emph{base rectangle} is $B_{n+1} = \prod_{k=1}^d [0, b^{(k)}_{n+1}]$.
\end{definition}

The base vector depends on $e_{n+1}$ only through the rank of $e_{n+1}^{(k)}$
relative to the calibration values.  Specifically,
\begin{equation}\label{eq:bk-formula}
  b^{(k)}_{n+1} =
  \begin{cases}
    L_k             & \text{if } e_{n+1}^{(k)} \leq L_k, \\
    e_{n+1}^{(k)}   & \text{if } L_k < e_{n+1}^{(k)} \leq U_k, \\
    U_k             & \text{if } e_{n+1}^{(k)} > U_k,
  \end{cases}
\end{equation}
so $L_k \leq b^{(k)}_{n+1} \leq U_k$ for all $e_{n+1}$.
\figureref{fig:cells2d} illustrates the thresholds $L_k$ and $U_k$, the
resulting cell structure for $d = 2$, and the clamping relationship
between $e_{n+1}$ and $b_{n+1}$.

\begin{figure}[t]
\floatconts
  {fig:scoring-and-cells}%
  {\caption{Scaling score and cell decomposition ($d = 2$).}}%
  {%
    \subfigure[Scaling score for three residuals ($d=2$, base vector
      $(b_1,b_2)$).  The score $\sigma$ is the smallest factor by which the
      base rectangle must be scaled to contain the point.  All
      scaled-rectangle corners lie on the ray through $(b_1,b_2)$ (gray
      line).]{\label{fig:scoring}%
      \includegraphics[width=0.48\linewidth]{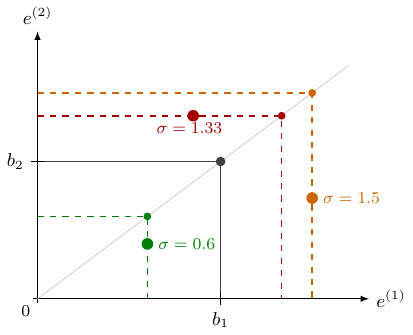}}%
    \hfill
    \subfigure[The $3^2 = 9$ cells of the decomposition for $d = 2$.
      The thresholds $L_1, U_1$ and $L_2, U_2$ divide $\mathbb{R}_+^2$
      into a $3 \times 3$ grid.  Corner cells (blue) have
      $\Imid = \varnothing$; mixed cells (red) have $|\Imid| \geq 1$.
      The clamping relationship between $e_{n+1}$ and $b_{n+1}$ is also
      visible: each coordinate of $b_{n+1}$ equals the clamp of the
      corresponding coordinate of $e_{n+1}$ to the interval $[L_k, U_k{]}$.]{\label{fig:cells2d}%
      \includegraphics[width=0.50\linewidth]{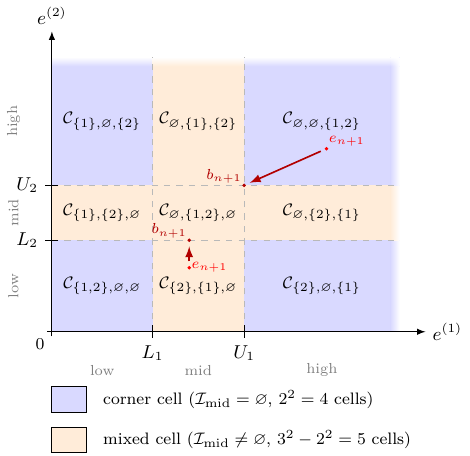}}%
  }
\end{figure}

In the transductive setting, all $n+1$ scores share the same base $b_{n+1}$:
\[
  \sigma_i(e_{n+1}) = \sigma(e_i;\,b_{n+1}),\quad i=1,\ldots,n+1.
\]

The following theorem is the main coverage guarantee of the method.  It states
that the prediction region defined via a conformal quantile achieves the
desired joint marginal coverage, under the sole assumption of exchangeability.

\begin{theorem}[Validity]\label{prop:validity}
Let $q = \lceil(1-\alpha)(n+1)\rceil$ and let $\hatlam(e_{n+1})$ denote the
$q$-th smallest value in $\{\sigma_1(e_{n+1}),\ldots,\sigma_n(e_{n+1}),+\infty\}$.
Define the prediction region in error space
\[
  \calR_\alpha
  = \bigl\{e_{n+1} \in \mathbb{R}_+^d :
      \sigma_{n+1}(e_{n+1}) \leq \hatlam(e_{n+1})\bigr\}.
\]
If $(e_1,\ldots,e_n,e_{n+1})$ are exchangeable, then
$\mathbb{P}(e_{n+1} \in \calR_\alpha) \geq 1-\alpha$.
\end{theorem}
\begin{proof}
See the \hyperref[proof:validity]{proof} in \appendixref{app:proofs}.
\end{proof}

The validity guarantee in \propositionref{prop:validity} establishes that $\calR_\alpha$
achieves the desired coverage, but the definition
$\calR_\alpha = \{e_{n+1} : \sigma_{n+1}(e_{n+1}) \leq \hatlam(e_{n+1})\}$
is an \emph{implicit} equation: the construction is \emph{transductive}, meaning
the base vector $b_{n+1}$ depends on $e_{n+1}$ through the clamping
formula~\eqref{eq:bk-formula}, so every candidate test error induces a
different calibration ranking and, potentially, a different quantile threshold.
Computing $\calR_\alpha$ explicitly therefore requires a dedicated analysis.  To handle this, we decompose $\mathbb{R}_+^d$ into a finite collection of
\emph{cells} and treat each cell separately: conditioned on the test error
$e_{n+1}$ belonging to a given cell, the base vector $b_{n+1}$ is constant
throughout that cell, so the calibration scores and the conformal quantile can
be computed analytically.

\subsection{Cell decomposition and scores}
\label{sec:method-cells}

The thresholds $L_k$ and $U_k$ partition the $k$-th coordinate axis into three
intervals: below $L_k$ (lower region), between $L_k$ and $U_k$ (middle
region), and above $U_k$ (upper region).  Taking the product over all $d$
coordinates yields a partition of $\mathbb{R}_+^d$ into $3^d$ cells.  Each
cell is characterised by specifying, for every coordinate $k \in [d]$, whether
that coordinate falls in the lower, middle, or upper interval.  We encode this
by a triple $(\Ilow, \Imid, \Ihigh)$, a partition of $[d]$, where $\Ilow$
collects the coordinates in the lower region, $\Imid$ those in the middle
region, and $\Ihigh$ those in the upper region.  \figureref{fig:cells2d} illustrates
the $3^2 = 9$ cells for $d = 2$.  This is captured formally by the following
definition.

\begin{definition}[Cells]\label{def:cell}
For a partition $(\Ilow, \Imid, \Ihigh)$ of $[d]$, the corresponding
\emph{cell} is
% Old single-line version (kept for reference):
% \calC_{\Ilow,\Imid,\Ihigh} = \bigl\{e \in \mathbb{R}_+^d :
%   e^{(k)} \leq L_k\;\forall k \in \Ilow,\quad
%   L_k < e^{(k)} \leq U_k\;\forall k \in \Imid,\quad
%   e^{(k)} > U_k\;\forall k \in \Ihigh\bigr\}.
\[
  \calC_{\Ilow,\Imid,\Ihigh} = \left\{e \in \mathbb{R}_+^d :
  \begin{array}{ll}
    e^{(k)} \leq L_k, & \forall k \in \Ilow,\\[2pt]
    L_k < e^{(k)} \leq U_k, & \forall k \in \Imid,\\[2pt]
    e^{(k)} > U_k, & \forall k \in \Ihigh
  \end{array}
  \right\}.
\]
\end{definition}

Assuming the test error $e_{n+1}$ belongs to cell $\calC_{\Ilow,\Imid,\Ihigh}$,
the base vector~\eqref{eq:bk-formula} is fully determined by the partition:
$b^{(k)}_{n+1} = L_k$ for $k \in \Ilow$,
$b^{(k)}_{n+1} = e_{n+1}^{(k)}$ for $k \in \Imid$, and
$b^{(k)}_{n+1} = U_k$ for $k \in \Ihigh$.
\figureref{fig:cells2d} illustrates how $b_{n+1}$ depends on $e_{n+1}$ in the
three-region partition for each coordinate.

Cells with $\Imid = \varnothing$ are called \emph{corner cells}: the base
vector does not depend on the test error throughout the cell, and we denote
its value $b_{\Ilow,\Ihigh}$.  Cells with $\Imid \neq \varnothing$ are
\emph{mixed cells}.  There are $2^d$ corner cells and $3^d - 2^d$ mixed cells.
\figureref{fig:cells2d} illustrates both types for $d = 2$.

Given the scaling score $\sigma(e;\,b) = \max_k e^{(k)}/b^{(k)}$ and the
cell formula for $b_{n+1}$, the scores $\sigma_{n+1}$ and $\sigma_i$ take a
simple explicit form on each cell, as summarised in the following lemma.

\begin{lemma}[Scores in any cell]\label{lem:scores}
Let $e_{n+1} \in \calC_{\Ilow,\Imid,\Ihigh}$.
With the convention $\max_\varnothing:=0$:
\begin{align}
  \sigma_{n+1}(e_{n+1}) &= \max\Bigl(
    \max_{k \in \Ilow}\frac{e_{n+1}^{(k)}}{L_k},\;
    \max_{k \in \Imid} 1,\;
    \max_{k \in \Ihigh}\frac{e_{n+1}^{(k)}}{U_k}
  \Bigr), \label{eq:test-score}\\[4pt]
  \sigma_i(e_{n+1}) &= \max\Bigl(
    \max_{k \in \Ilow}\frac{e_i^{(k)}}{L_k},\;
    \max_{k \in \Imid}\frac{e_i^{(k)}}{e_{n+1}^{(k)}},\;
    \max_{k \in \Ihigh}\frac{e_i^{(k)}}{U_k}
  \Bigr), \quad i \leq n. \label{eq:cal-score}
\end{align}
When $e_{n+1}$ belongs to a corner cell ($\Imid = \varnothing$),
$\sigma_i(e_{n+1})$ is constant on the cell.
When $e_{n+1}$ belongs to a mixed cell ($\Imid \neq \varnothing$),
$\sigma_{n+1}(e_{n+1}) \geq 1$, and the mid terms in
$\sigma_i(e_{n+1})$ are non-increasing in each $e_{n+1}^{(k)}$, $k \in \Imid$.
\end{lemma}
\begin{proof}
See the \hyperref[proof:scores]{proof} in \appendixref{app:proofs}.
\end{proof}

\subsection{Prediction region characterisation}
\label{sec:method-region}

Having partitioned $\mathbb{R}_+^d$ into cells, we now characterise
$\calR_\alpha$ by working one cell at a time.  The key observation is the
following: \emph{assuming the test error $e_{n+1}$ belongs to a given cell},
the base vector $b_{n+1}$ is fully determined and remains constant for all
test errors confined to that cell (in corner cells) or depends on $e_{n+1}$
only through an explicit formula (in mixed cells).  As a consequence, all
calibration scores $\{\sigma_i(e_{n+1})\}_{i=1}^n$ are determined, the
conformal quantile $\hatlam(b_{n+1})$ can be computed analytically.  For
corner cells, the base vector is fixed and the region boundary has an explicit
closed form.  For mixed cells, the base vector still depends on $e_{n+1}$
through its mid coordinates, so the boundary retains an implicit form; we
characterise it analytically and provide sandwich bounds.  We treat corner
cells and mixed cells in turn.

\subsubsection{Corner cells}

When the test error $e_{n+1}$ belongs to a corner cell
$\calC_{\Ilow,\varnothing,\Ihigh}$, the base vector $b_{\Ilow,\Ihigh}$ does not depend on the test error, so all calibration scores
\[
  \sigma_i^{\Ilow,\Ihigh}
  = \max\Bigl(
      \max_{k \in \Ilow}\frac{e_i^{(k)}}{L_k},\;
      \max_{k \in \Ihigh}\frac{e_i^{(k)}}{U_k}
    \Bigr)
\]
are independent of $e_{n+1}$, and the conformal quantile
$\hatlam_{\Ilow,\Ihigh}$ is a fixed number.

\begin{proposition}[Exact region on corner cells]\label{prop:corner-exact}
Let $\hatlam = \hatlam_{\Ilow,\Ihigh}$.
The exact prediction region restricted to the corner cell is
\begin{equation}\label{eq:corner-region}
  \calR_\alpha \cap \calC_{\Ilow,\varnothing,\Ihigh}
  =
  \prod_{k \in \Ilow} \bigl[0,\,\min(\hatlam,1)\,L_k\bigr]
  \;\times\;
  \prod_{k \in \Ihigh} \bigl(U_k,\,\hatlam\,U_k\bigr],
\end{equation}
with the convention that $(U_k,\,\hatlam\,U_k] = \varnothing$
when $\hatlam \leq 1$.
\end{proposition}
\begin{proof}
See the \hyperref[proof:corner-exact]{proof} in \appendixref{app:proofs}.
\end{proof}

Two boundary cases of \propositionref{prop:corner-exact} are worth noting.
When $\Ihigh \neq \varnothing$, the intersection~\eqref{eq:corner-region} is
non-empty if and only if $\hatlam > 1$: the conformal region extends beyond
$U_k$ only when the quantile exceeds $1$.
When $\Ihigh = \varnothing$ (all-low corner cell), the region always contains
the origin and equals the full cell $\prod_{k=1}^d [0, L_k]$ if and only if
$\hatlam \geq 1$.

\subsubsection{Mixed cells}

When the test error $e_{n+1}$ belongs to a mixed cell
$\calC_{\Ilow,\Imid,\Ihigh}$ (with $\Imid \neq \varnothing$), the base vector
$b_{n+1}$ depends on $e_{n+1}$ through the mid coordinates
(since $b_{n+1}^{(k)} = e_{n+1}^{(k)}$ for $k \in \Imid$).  As a consequence,
the calibration scores $\{\sigma_i(e_{n+1})\}$ and the conformal quantile also
depend on $e_{n+1}$, but only through its $\Imid$ coordinates; we write
$\hatlam(e_{n+1}^{(\Imid)})$ to make this dependence explicit.

\begin{proposition}[Exact region on mixed cells]\label{prop:mixed-exact}
Assuming the test error $e_{n+1}$ belongs to a mixed cell
$\calC_{\Ilow,\Imid,\Ihigh}$:
\begin{enumerate}
\item If $\Ihigh = \varnothing$: $\sigma_{n+1}(e_{n+1}) = 1$, and
\[
  \calR_\alpha \cap \calC_{\Ilow,\Imid,\varnothing}
  = \bigl\{e_{n+1} \in \calC_{\Ilow,\Imid,\varnothing} :
    \hatlam(e_{n+1}^{(\Imid)}) \geq 1\bigr\}.
\]
\item If $\Ihigh \neq \varnothing$: $\sigma_{n+1}(e_{n+1}) =
  \max_{k \in \Ihigh} e_{n+1}^{(k)}/U_k$, and
\[
  \calR_\alpha \cap \calC_{\Ilow,\Imid,\Ihigh}
  = \bigl\{e_{n+1} \in \calC_{\Ilow,\Imid,\Ihigh} :
    e_{n+1}^{(k)} \leq \hatlam(e_{n+1}^{(\Imid)})\,U_k\;\forall k \in \Ihigh\bigr\}.
\]
\end{enumerate}
In both cases, the low coordinates are unconstrained within the cell: for
$k \in \Ilow$, cell membership already forces $e_{n+1}^{(k)} \leq L_k =
b_{n+1}^{(k)}$, so the low terms $e_{n+1}^{(k)}/b_{n+1}^{(k)} \leq 1$ never
drive $\sigma_{n+1}$, and $\sigma_{n+1} \leq \hatlam$ imposes no further
constraint on them beyond $e_{n+1}^{(k)} \leq L_k$.
\end{proposition}
\begin{proof}
See the \hyperref[proof:mixed-exact]{proof} in \appendixref{app:proofs}.
\end{proof}

Item~1 of \propositionref{prop:mixed-exact} ($\Ihigh = \varnothing$) reduces to the
single condition $\hatlam(e_{n+1}^{(\Imid)}) \geq 1$: since $\sigma_{n+1} = 1$
throughout this cell, inclusion in $\calR_\alpha$ depends only on whether the
conformal quantile reaches~$1$.  In item~2 ($\Ihigh \neq \varnothing$), the
condition additionally constrains the high coordinates via
$\hatlam(e_{n+1}^{(\Imid)})$.  Both cases are \emph{implicit}: $\hatlam(e_{n+1}^{(\Imid)})$
depends on $e_{n+1}$ through the mid coordinates, and there is in general no
closed-form description of $\calR_\alpha$ on mixed cells.
Nevertheless, $\hatlam(e_{n+1}^{(\Imid)})$ can be bracketed analytically,
as the following proposition shows.

\begin{proposition}[Monotonicity and interpolation]%
\label{prop:mixed-monotone}
On a mixed cell $\calC_{\Ilow,\Imid,\Ihigh}$:
\begin{enumerate}
\item $\hatlam(e_{n+1}^{(\Imid)})$ is non-increasing in each mid coordinate.
\item At the boundary $e_{n+1}^{(k)} \to L_k^+$ for all $k \in \Imid$:
  $\hatlam \to \hatlam_{\Ilow \cup \Imid,\,\Ihigh}$.
\item At $e_{n+1}^{(k)} = U_k$ for all $k \in \Imid$:
  $\hatlam = \hatlam_{\Ilow,\,\Imid \cup \Ihigh}$.
\end{enumerate}
Hence $\hatlam_{\Ilow,\,\Imid \cup \Ihigh}
\leq \hatlam(e_{n+1}^{(\Imid)})
\leq \hatlam_{\Ilow \cup \Imid,\,\Ihigh}$ throughout the cell.
\end{proposition}
\begin{proof}
See the \hyperref[proof:mixed-monotone]{proof} in \appendixref{app:proofs}.
\end{proof}

\propositionref{prop:mixed-exact,prop:mixed-monotone} together imply that the boundary
of $\calR_\alpha$ is continuous across cell boundaries: items~2 and~3 of
\propositionref{prop:mixed-monotone} show that $\hatlam$ matches the corner-cell values
at each face, so no jump can occur when crossing from one cell to an adjacent one.

Together, \propositionref{prop:corner-exact,prop:mixed-exact,prop:mixed-monotone} give a
complete analytical picture of $\calR_\alpha$: on corner cells it reduces to an
explicit rectangle, and on mixed cells it is characterised by the two-item
structure of \propositionref{prop:mixed-exact}.  While $\calR_\alpha$ is not itself a
hyperrectangle, it can be sandwiched between two axis-aligned hyperrectangles
(the inner and outer rectangles of \sectionref{sec:method-three-modes}), making it
both principled and tractable.

\subsubsection{Structural properties}

Assembling the cell-by-cell characterisations above gives the full prediction
region $\calR_\alpha = \bigcup_{(\Ilow,\Imid,\Ihigh)} \bigl(\calR_\alpha \cap
\calC_{\Ilow,\Imid,\Ihigh}\bigr)$.  A key global property of this union is that
it is downward-closed with respect to the component-wise order.  Intuitively,
if a residual $e_{n+1}$ is already within the conformal region, then any
smaller residual $e' \preceq e_{n+1}$ must also lie within it.  This follows
from the monotonicity of the score ratio $x \mapsto x/\operatorname{clamp}(x,L,U)$,
which is non-decreasing on $\mathbb{R}_+$ for any fixed $0 < L \leq U$ (the
proof is immediate by case analysis on the three branches of the clamp).

\begin{proposition}[Downward-closedness]\label{prop:downward-closed}
$\calR_\alpha$ is downward-closed with respect to $\preceq$:
if $e_{n+1} \in \calR_\alpha$ and $e' \preceq e_{n+1}$, then
$e' \in \calR_\alpha$.
\end{proposition}
\begin{proof}
See the \hyperref[proof:downward-closed]{proof} in \appendixref{app:proofs}.
\end{proof}

\subsection{Nested prediction regions}
\label{sec:method-three-modes}

From a single calibration run, our method produces four related prediction
regions with different trade-offs between validity, tightness, and computational
cost (\propositionref{prop:sandwich}); see \figureref{fig:regions} for an illustration.

\begin{figure}[t]
\floatconts
  {fig:cells-and-regions}%
  {\caption{Effect of $\gamma$ on the prediction region shape, and the four
    nested region types ($d = 2$).}}%
  {%
    \subfigure[Effect of $\gamma$ on $\calR_\alpha$ for fixed $1-\alpha$ ($d=2$).
      The dashed box is the rectangle $\prod_k[0,U_k{]}$, which grows with $\gamma$.
      \emph{Left} ($\gamma < \gamma_{\mathrm{rect}}$): base rectangle too small,
      $\hat\lambda > 1$, $\calR_\alpha$ extends beyond it.
      \emph{Centre} ($\gamma = \gamma_{\mathrm{rect}}$): $\hat\lambda = 1$,
      $\calR_\alpha$ coincides with the base rectangle.
      \emph{Right} ($\gamma > \gamma_{\mathrm{rect}}$): $\hat\lambda < 1$,
      $\calR_\alpha$ strictly inside the base rectangle.]{\label{fig:gamma-effect}%
      \includegraphics[width=0.58\linewidth]{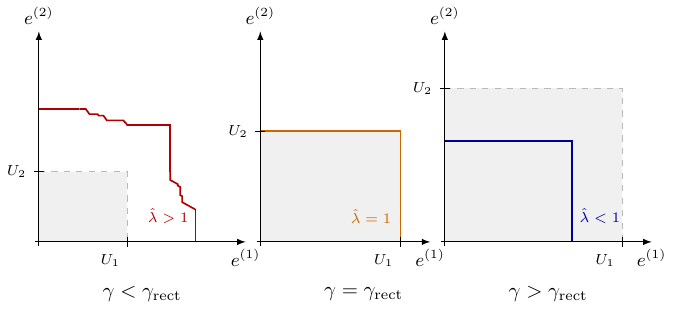}}%
    \hfill
    \subfigure[Prediction boundaries for $d = 2$.
      The exact set $\calR_\alpha$ (red, zigzag) and its staircase
      approximation (green) are both valid.  The outer rectangle (blue)
      is conservative.  The inner rectangle (violet, dashed) lacks a
      coverage guarantee.  Boundaries are slightly offset for visual
      clarity.]{\label{fig:regions}%
      \includegraphics[width=0.40\linewidth]{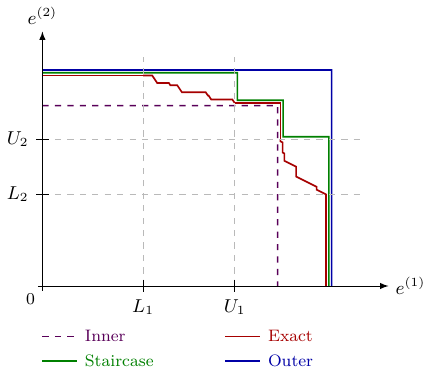}}%
  }
\end{figure}
The \emph{inner rectangle} is the most compact prediction region our method
produces.  It is contained inside $\calR_\alpha$ but carries no formal coverage
guarantee.  The \emph{exact region} $\calR_\alpha$ is the tightest valid
prediction set, but it is defined implicitly and costly to characterise on
mixed cells.  The \emph{staircase} is a valid over-approximation of $\calR_\alpha$
given as an explicit union of at most $2^d$ hyperrectangles, each computable at
calibration time using the corner-cell quantiles.  The \emph{outer rectangle}
is the most conservative: a single valid hyperrectangle, but the largest of
the four.

\paragraph{Outer rectangle (SCO) --- valid.}
The outer rectangle is obtained by being \emph{doubly conservative}: scores
are computed relative to the lower base rectangle ($L_k$ sides), which yields
the \emph{largest} possible calibration scores and therefore the \emph{largest}
possible quantile
$\hatlam_{\mathrm{ub}}$; it then builds the confidence region using the
\emph{largest} possible base rectangle ($U_k$ sides).  The resulting
hyperrectangle is guaranteed to contain $\calR_\alpha$ and hence achieves
valid joint coverage, at the cost of being conservative.
Assume $L_k > 0$ for all $k$.  Formally, define upper-bound scores
\[
  \sigma_i^{\mathrm{ub}} = \max_{k \in [d]}\frac{e_i^{(k)}}{L_k},
  \qquad
  \hatlam_{\mathrm{ub}} = \text{$q$-th order statistic of }
    \{\sigma_i^{\mathrm{ub}}\}_{i=1}^n.
\]

\begin{proposition}[Rectangular outer bound]\label{prop:outer-bound}
Assume $U_k > 0$ for all $k \in [d]$.
The exact prediction region is contained in the following closed-form
hyperrectangle:
\begin{equation}\label{eq:outer-bound}
  \calR_\alpha \;\subseteq\; \prod_{k=1}^d \bigl[0,\,\hatlam_{\mathrm{ub}}\,U_k\bigr].
\end{equation}
The region $\prod_k [0, \hatlam_{\mathrm{ub}} U_k]$ achieves joint coverage
$\geq 1-\alpha$.
\end{proposition}
\begin{proof}
See the \hyperref[proof:outer-bound]{proof} in \appendixref{app:proofs}.
\end{proof}
\algorithmref{alg:outer} gives the calibration and prediction procedures for the outer
rectangle.

\begin{algorithm2e}[t]
\caption{SCO: Scaling-Score Outer-Bound Predictor}
\label{alg:outer}
\KwIn{Calibration errors $\{e_i\}_{i=1}^n \subset \mathbb{R}_+^d$,
      miscoverage level $\alpha \in (0,1)$,
      base-rectangle level $\gamma \in (0,1)$}
\KwOut{Half-widths $h_1,\ldots,h_d \geq 0$}
\textit{--- Calibration phase (run once) ---}\\
$m' \leftarrow \operatorname{clip}(\lceil\gamma(n+1)\rceil,\;2,\;n)$\;
\For{$k = 1, \ldots, d$}{
  $L_k \leftarrow e_{(m'-1)}^{(k)}$,\quad $U_k \leftarrow e_{(m')}^{(k)}$\;
}
\For{$i = 1, \ldots, n$}{
  $\sigma_i^{\mathrm{ub}} \leftarrow \max_{k \in [d]} e_i^{(k)} / L_k$\;
}
$q \leftarrow \lceil(1-\alpha)(n+1)\rceil$\;
$\hatlam_{\mathrm{ub}} \leftarrow \text{$q$-th order statistic of } \{\sigma_i^{\mathrm{ub}}\}_{i=1}^n$\;
\For{$k = 1, \ldots, d$}{
  $h_k \leftarrow \hatlam_{\mathrm{ub}} \cdot U_k$\;
}
\textit{--- Prediction phase (per test point $X_{n+1}$) ---}\\
Compute $\hat{y} \leftarrow \hat{f}(X_{n+1})$\;
\Return{$\prod_{k=1}^d [\hat{y}^{(k)} - h_k,\;\hat{y}^{(k)} + h_k]$}
\end{algorithm2e}

\paragraph{Inner rectangle (SCI) --- no formal validity guarantee.}
The inner rectangle is doubly optimistic: scores are computed relative to the
upper base rectangle ($U_k$ sides), which yields the \emph{smallest} possible
calibration scores, giving the lowest
possible quantile $\hatlam_{\mathrm{lb}}$; it then uses the \emph{smallest}
possible base rectangle ($L_k$ sides) to build the confidence region. The
result is tight but not guaranteed to be valid.  Formally, define lower-bound scores
\[
  \sigma_i^{\mathrm{lb}} = \max_{k \in [d]}\frac{e_i^{(k)}}{U_k},
  \qquad
  \hatlam_{\mathrm{lb}} = \text{$q$-th order statistic of }
    \{\sigma_i^{\mathrm{lb}}\}_{i=1}^n.
\]
\algorithmref{alg:inner} gives the calibration and prediction procedures for the inner
rectangle.

\begin{algorithm2e}[t]
\caption{SCI: Scaling-Score Inner-Bound Predictor}
\label{alg:inner}
\KwIn{Calibration errors $\{e_i\}_{i=1}^n \subset \mathbb{R}_+^d$,
      miscoverage level $\alpha \in (0,1)$,
      base-rectangle level $\gamma \in (0,1)$}
\KwOut{Half-widths $h_1^{\mathrm{lb}},\ldots,h_d^{\mathrm{lb}} \geq 0$
       (\emph{not} a valid prediction region)}
\textit{--- Calibration phase ---}\\
Compute $m', L_k, U_k$ as in \algorithmref{alg:outer}\;
\For{$i = 1, \ldots, n$}{
  $\sigma_i^{\mathrm{lb}} \leftarrow \max_{k \in [d]} e_i^{(k)} / U_k$\;
}
$q \leftarrow \lceil(1-\alpha)(n+1)\rceil$\;
$\hatlam_{\mathrm{lb}} \leftarrow \text{$q$-th order statistic of } \{\sigma_i^{\mathrm{lb}}\}_{i=1}^n$\;
\For{$k = 1, \ldots, d$}{
  $h_k^{\mathrm{lb}} \leftarrow \hatlam_{\mathrm{lb}} \cdot L_k$\;
}
\end{algorithm2e}

\paragraph{Exact region $\calR_\alpha$ (valid, tightest but implicit).}
The exact prediction region $\calR_\alpha$ is characterised cell by cell via
\propositionref{prop:corner-exact,prop:mixed-exact}: on corner cells and all-low-mid mixed
cells it has an explicit form, but on mixed cells with $\Ihigh \neq \varnothing$
the boundary is only implicitly defined through $\hatlam(e_{n+1}^{(\Imid)})$,
which depends on the test error.
Although $\calR_\alpha$ is the tightest possible valid region, its boundary
forms broken decision lines that are difficult to characterise explicitly and
whose computation requires evaluating all calibration scores for every query.

\begin{algorithm2e}[t]
\caption{$\mathrm{SC}^2$: Staircase Construction (Calibration Phase)}
\label{alg:staircase}
\KwIn{Calibration errors $\{e_i\}_{i=1}^n \subset \mathbb{R}_+^d$,
      miscoverage level $\alpha \in (0,1)$,
      base-rectangle level $\gamma \in (0,1)$}
\KwOut{Pareto-optimal staircase rectangles $\{[0, e^*_S]\}_{S \in \mathcal{P}}$}
Compute $m', U_k$ as in \algorithmref{alg:outer}\;
$q \leftarrow \lceil(1-\alpha)(n+1)\rceil$\;
\For{each $S \subseteq [d]$}{
  Compute scores $\sigma_i^{\bar{S},S} \leftarrow \max_{k \in S} e_i^{(k)} / U_k$
  for $i = 1,\ldots,n$\;
  $\hatlam_{\bar{S},S} \leftarrow \text{$q$-th order statistic of }
  \{\sigma_i^{\bar{S},S}\}_{i=1}^n$\;
  Set corner point $e^{*(k)}_S \leftarrow
  \begin{cases} \hatlam_{\bar{S},S}\,U_k & k \in S, \\ U_k & k \notin S. \end{cases}$\;
}
Retain only Pareto-optimal rectangles:
$\mathcal{P} \leftarrow \{S \subseteq [d] : \nexists\, S' \text{ s.t. }
e^{*(k)}_{S'} \leq e^{*(k)}_S\ \forall k,\ \text{strictly for some } k\}$\;
\Return{$\{[0, e^*_S]\}_{S \in \mathcal{P}}$}
\end{algorithm2e}

\paragraph{Staircase ($\mathrm{SC}^2$) --- valid, union of at most $2^d$ rectangles.}
The staircase is a valid over-approximation of $\calR_\alpha$ that replaces the
implicit mixed-cell boundary by the upper-bounding corner-cell quantile from
\propositionref{prop:mixed-monotone}.  This turns every cell boundary into a fixed
rectangle that is independent of the test error.  The result is an explicit
union of $2^d$ hyperrectangles, one per subset $S \subseteq [d]$, all
computable at calibration time (\algorithmref{alg:staircase}).  The staircase is valid since it contains
$\calR_\alpha$, and it is tight in the sense that the $2^d$ rectangles each
touch the boundary of $\calR_\alpha$ at their corner points, as the following proposition shows.

\begin{proposition}[Staircase: union of hyperrectangles]\label{prop:extremal}
Assume $U_k > 0$ for all $k \in [d]$.
For each $S \subseteq [d]$, define the corner point
\[
  e^{*(k)}_S =
  \begin{cases}
    \hatlam_{\bar{S},S}\,U_k & k \in S, \\
    U_k                       & k \notin S,
  \end{cases}
\]
where $\bar{S} = [d]\setminus S$ and $\hatlam_{\bar{S},S}$ is the conformal
quantile of the corner cell $\calC_{\bar{S},\varnothing,S}$.
Then
\[
  \calR_\alpha \;\subseteq\; \bigcup_{S \subseteq [d]}
  \prod_{k=1}^d [0,\,e^{*(k)}_S].
\]
\end{proposition}
\begin{proof}
See the \hyperref[proof:extremal]{proof} in \appendixref{app:proofs}.
\end{proof}

The staircase $\bigcup_S \prod_k [0, e^{*(k)}_S]$ is valid (it contains
$\calR_\alpha$), and all $2^d$ corner rectangles are computable at calibration
time: each $\hatlam_{\bar{S},S}$ requires one pass over calibration scores.
In practice, a simple post-processing step retains only the Pareto-optimal
rectangles (those not dominated component-wise by any other), which often
reduces the number of active rectangles well below $2^d$ (empirically
quantified in \sectionref{sec:exp-results}, \tableref{tab:real_world_frac_retained}).

The four prediction regions are related by the following chain of containments.

\begin{proposition}[Rectangular sandwich]\label{prop:sandwich}
Assume $U_k > 0$ for all $k \in [d]$.
Combining \propositionref{prop:outer-bound} with the cell-based analysis, whenever
$\hatlam_{\mathrm{ub}} \geq 1$ (the regime that arises in practice), the
four prediction regions form the nested chain
\begin{equation}\label{eq:sandwich-full}
  \underbrace{\prod_{k=1}^d \bigl[0,\,\hatlam_{\mathrm{lb}}\,L_k\bigr]}_{\text{inner rectangle}}
  \;\subseteq\;
  \underbrace{\calR_\alpha}_{\text{exact region}}
  \;\subseteq\;
  \underbrace{\bigcup_{S \subseteq [d]}\prod_{k=1}^d [0,\,e^{*(k)}_S]}_{\text{staircase}}
  \;\subseteq\;
  \underbrace{\prod_{k=1}^d \bigl[0,\,\hatlam_{\mathrm{ub}}\,U_k\bigr]}_{\text{outer rectangle}}.
\end{equation}
In the edge case $\hatlam_{\mathrm{ub}} < 1$, the staircase need not be
contained in the outer rectangle, but the inner sandwich
\begin{equation}\label{eq:sandwich}
  \prod_{k=1}^d \bigl[0,\,\hatlam_{\mathrm{lb}}\,L_k\bigr]
  \;\subseteq\;
  \calR_\alpha
  \;\subseteq\;
  \bigcup_{S \subseteq [d]}\prod_{k=1}^d [0,\,e^{*(k)}_S]
\end{equation}
still holds, and \propositionref{prop:outer-bound} provides
$\calR_\alpha \subseteq \prod_k [0,\,\hatlam_{\mathrm{ub}}\,U_k]$ as a
separate over-approximation.
\end{proposition}
\begin{proof}
See the \hyperref[proof:sandwich]{proof} in \appendixref{app:proofs}.
\end{proof}

The inner rectangle is the most compact but offers no coverage guarantee; it
may under-cover since it is strictly contained in $\calR_\alpha$.  Both the
staircase and the outer rectangle are valid prediction regions (they contain
$\calR_\alpha$).  The full chain~\eqref{eq:sandwich-full} holds in the typical
regime $\hatlam_{\mathrm{ub}} \geq 1$.  When $\hatlam_{\mathrm{ub}} < 1$, which
can occur at large $\gamma$, the outer rectangle shrinks below $[0, U]$,
while the staircase still contains $[0, U]$ via its empty-subset corner
$e^*_\varnothing = U$, so the two regions become incomparable.
These four objects are illustrated in \figureref{fig:regions}.

\subsection{Choice of \texorpdfstring{$\gamma$}{gamma}}
\label{sec:gamma-choice}

The hyperparameter $\gamma \in (0,1)$ determines the marginal quantile level at
which the thresholds $U_k = e_{(m')}^{(k)}$ are set.  Choosing $\gamma$ too
small or too large both degrade performance, though for different reasons.

When $\gamma$ is small, the thresholds $U_k = e_{(m')}^{(k)}$ sit at low
marginal quantiles of the calibration errors.  The base rectangle then
under-represents the scale of large errors on each target: the outer
quantile $\hatlam$ must compensate by expanding the region uniformly,
which dilutes the per-target scale adaptation that is the key advantage of
the method.

When $\gamma$ is large, the thresholds $U_k$ are set by the most extreme
calibration errors, making them sensitive to outliers.  Moreover, once
$\gamma > \gamma_{\mathrm{rect}}$ the base rectangle already covers more than
$1-\alpha$ of calibration points, so $\hatlam < 1$ and the prediction region
shrinks to a rectangle strictly inside the base rectangle, losing the
non-rectangular staircase structure.

We define $\gamma_{\mathrm{rect}}$ as the value of $\gamma$ for which the base
rectangle covers exactly a fraction $1-\alpha$ of calibration errors.  The
base rectangle $\prod_k [0, U_k]$ has the equal-marginal-quantile property:
each side sits at ECDF level $\gamma$.  Because joint domination is strictly
harder than marginal domination when targets are not perfectly co-monotone,
$\gamma_{\mathrm{rect}} \geq 1-\alpha$ in general, with equality only under
perfect positive dependence.  When $\gamma = \gamma_{\mathrm{rect}}$,
$\hatlam_{\mathrm{lb}} \approx 1$ and the inner rectangle approximates the
empirical copula rectangle (with equality in the limit $L_k \to U_k$,
i.e.\ when the $m'$-th and $(m'-1)$-th calibration order statistics
coincide), while the outer rectangle provides a finitely valid conformal
guarantee that the empirical copula lacks.

We recommend targeting $\gamma < \gamma_{\mathrm{rect}}$, so that the staircase
retains its non-rectangular structure and $\hatlam > 1$ provides a meaningful
scale signal.  The choice $\gamma = 1-\alpha$ is a natural default: it is
tied to the coverage level, always satisfies $\gamma \leq \gamma_{\mathrm{rect}}$
in practice, and our experiments confirm that it consistently matches or
outperforms other values of $\gamma$ across datasets and dimensions
(see \sectionref{sec:exp-results}).

%% ==========================================================================
\section{Experiments}
\label{sec:experiments}
%% ==========================================================================

\subsection{Protocol}
\label{sec:exp-protocol}

We use a random forest base model (100 trees) throughout.  Each dataset is
split into 60\% train / 15\% calibration / 25\% test.  Results are averaged
over 20 random seeds.  We evaluate at $\alpha \in \{0.05, 0.10\}$.

\paragraph{Baselines.}
\begin{itemize}
  \item \textbf{Max}: max-aggregation score producing an axis-aligned hypercube.
  \item \textbf{CHR} \citep{sampson_conformal_2024}: valid rectangular regions
        with per-target width adaptation, at the cost of splitting the
        calibration set in two.
  \item \textbf{TSCP} \citep{fan_interpretable_2025}: Transductively
        Standardized Conformal Prediction, valid and model-agnostic with a
        single calibration set.
\end{itemize}

\paragraph{Our method.}
We evaluate the $\mathrm{SC}^2$ prediction set under three choices of the
hyperparameter $\gamma$, and SCO with $\gamma = 1-\alpha$:
\begin{itemize}
  \item $\mathrm{SC}^2$, $\gamma = 1-\alpha$: a natural default tied to the coverage level.
  \item $\mathrm{SC}^2$, $\gamma = 0.8$: a fixed, coverage-level-agnostic choice, and a lower bound on all $1-\alpha$ levels considered in our experiments ($0.80 \leq 0.90, 0.95, 0.99$).
  \item $\mathrm{SC}^2$, $\gamma_{\mathrm{rect}}$: chosen on the calibration set so that exactly
        a fraction $1-\alpha$ of calibration points are \emph{dominated} by the
        base rectangle (i.e.\ all their residuals fall within the base
        half-widths); see \sectionref{sec:gamma-choice}.
  \item SCO, $\gamma = 1-\alpha$: the closed-form valid outer rectangle.
\end{itemize}

\paragraph{Datasets.}
We evaluate on 29 real-world datasets spanning $d \in \{2, 3, 4, 6, 8, 12, 14,
16\}$ targets, covering a wide range of sample sizes and output dimensions.
See \appendixref{app:datasets} for the full list with sample sizes and sources.

\paragraph{Metrics.}
We evaluate each method on two metrics.
\emph{Joint coverage}: the fraction of test points whose full $d$-dimensional
response falls inside the prediction region (target: at least $1-\alpha$).
\emph{Volume}: the $\log_{10}$ volume of the prediction region, lower being
better for a fixed coverage level.

\subsection{Results}
\label{sec:exp-results}

\paragraph{Coverage.}
\tableref{tab:real_world_rf_cov} reports joint coverage at $\alpha=0.10$.  All
methods achieve valid coverage close to the nominal $1-\alpha = 0.90$ level
across all datasets, confirming the finite-sample guarantee.  The one
exception is $\mathrm{SC}^2$ with $\gamma=0.8$ on Births~2 ($d=4$): the random
forest predicts one target near-perfectly, so more than $80\%$ of calibration
residuals on that target are exactly zero, yielding $U_k = 0$ and violating
the assumption of \propositionref{prop:extremal,prop:outer-bound}.  The same degenerate
case arises for Solar~Flare~1 at $\alpha=0.05$ (see \appendixref{app:full-results}).
Choosing $\gamma \geq 1-\alpha$ (all other $\mathrm{SC}^2$ and SCO variants)
avoids this failure: if $\gamma \geq 1-\alpha$ and $U_k = 0$, then more than
$1-\alpha$ of calibration residuals are zero, the transductive quantile remains
finite, and coverage is maintained.

\begin{table}[htbp]
  \centering
  \small
  \setlength{\tabcolsep}{4pt}
  \begin{tabular}{l r r r r r r r r}
    \toprule
    Dataset & $d$ & \multicolumn{1}{c}{\multirow{2}{*}{Max}} & \multicolumn{1}{c}{\multirow{2}{*}{CHR}} & \multicolumn{1}{c}{\multirow{2}{*}{TSCP}} & \multicolumn{1}{c}{SCO} & \multicolumn{3}{c}{$\mathrm{SC}^2$} \\
    \cmidrule(lr){6-6} \cmidrule(lr){7-9}
    & &  &  &  & \multicolumn{1}{c}{$\gamma{=}1{-}\alpha$} & \multicolumn{1}{c}{$\gamma{=}0.8$} & \multicolumn{1}{c}{$\gamma{=}1{-}\alpha$} & \multicolumn{1}{c}{$\gamma_{\mathrm{rect}}$} \\
    \midrule
    ANSUR-II & 2 & $0.903$ & $0.897$ & $0.904$ & $0.908$ & $0.906$ & $0.906$ & $0.907$ \\
    Bio & 2 & $0.899$ & $0.900$ & $0.900$ & $0.900$ & $0.900$ & $0.900$ & $0.900$ \\
    Births-1 & 2 & $0.899$ & $0.897$ & $0.897$ & $0.896$ & $0.895$ & $0.895$ & $0.896$ \\
    Blog & 2 & $0.901$ & $0.901$ & $0.902$ & $0.901$ & $0.901$ & $0.901$ & $0.902$ \\
    CalCOFI & 2 & $0.899$ & $0.900$ & $0.900$ & $0.901$ & $0.900$ & $0.901$ & $0.901$ \\
    EDM & 2 & $0.905$ & $0.929$ & $0.929$ & $0.936$ & $0.927$ & $0.922$ & $0.906$ \\
    ENB & 2 & $0.915$ & $0.926$ & $0.908$ & $0.917$ & $0.910$ & $0.911$ & $0.922$ \\
    House & 2 & $0.900$ & $0.902$ & $0.900$ & $0.900$ & $0.899$ & $0.900$ & $0.901$ \\
    Taxi & 2 & $0.901$ & $0.902$ & $0.901$ & $0.901$ & $0.901$ & $0.901$ & $0.902$ \\
    Wage & 2 & $0.899$ & $0.902$ & $0.903$ & $0.902$ & $0.899$ & $0.901$ & $0.902$ \\
    \midrule
    Jura & 3 & $0.918$ & $0.937$ & $0.912$ & $0.914$ & $0.911$ & $0.910$ & $0.945$ \\
    SCPF & 3 & $0.902$ & $0.912$ & $0.909$ & $0.913$ & $0.909$ & $0.908$ & $0.911$ \\
    SF1 & 3 & $0.932$ & $0.946$ & $0.932$ & $0.935$ & $0.854$ & $0.925$ & $0.962$ \\
    SF2 & 3 & $0.905$ & $0.889$ & $0.906$ & $0.872$ & $0.868$ & $0.870$ & $0.907$ \\
    Student & 3 & $0.909$ & $0.904$ & $0.916$ & $0.917$ & $0.911$ & $0.910$ & $0.930$ \\
    \midrule
    Births-2 & 4 & $0.896$ & $0.906$ & $0.902$ & $0.905$ & $0.789$ & $0.905$ & $0.903$ \\
    Households & 4 & $0.899$ & $0.899$ & $0.901$ & $0.902$ & $0.901$ & $0.901$ & $0.902$ \\
    Stock & 4 & $0.900$ & $0.902$ & $0.911$ & $0.912$ & $0.908$ & $0.910$ & $0.922$ \\
    \midrule
    Air & 6 & $0.899$ & $0.899$ & $0.900$ & $0.898$ & $0.899$ & $0.898$ & $0.902$ \\
    ATP1d & 6 & $0.902$ & $0.908$ & $0.890$ & $0.908$ & $0.896$ & $0.899$ & $0.916$ \\
    ATP7d & 6 & $0.915$ & $0.929$ & $0.910$ & $0.921$ & $0.915$ & $0.910$ & $0.923$ \\
    \midrule
    RF1 & 8 & $0.899$ & $0.902$ & $0.899$ & $0.900$ & $0.900$ & $0.900$ & $0.899$ \\
    RF2 & 8 & $0.903$ & $0.903$ & $0.904$ & $0.905$ & $0.904$ & $0.905$ & $0.906$ \\
    \midrule
    OSales & 12 & $0.902$ & $0.907$ & $0.891$ & $0.901$ & $0.905$ & $0.899$ & $0.918$ \\
    \midrule
    WQ & 14 & $0.887$ & $0.898$ & $0.899$ & $0.902$ & $0.895$ & $0.897$ & $0.895$ \\
    \midrule
    OES10 & 16 & $0.914$ & $0.922$ & $0.897$ & $0.904$ & $0.900$ & $0.900$ & $0.910$ \\
    OES97 & 16 & $0.914$ & $0.923$ & $0.916$ & $0.914$ & $0.907$ & $0.907$ & $0.920$ \\
    SCM1d & 16 & $0.902$ & $0.902$ & $0.900$ & $0.900$ & $0.901$ & $0.900$ & $0.903$ \\
    SCM20d & 16 & $0.903$ & $0.905$ & $0.905$ & $0.905$ & $0.904$ & $0.905$ & $0.907$ \\
    \bottomrule
  \end{tabular}
  \caption{Joint coverage on real-world datasets (random-forest base model, $\alpha=0.10$). Mean over 20 seeds. Target: $\geq 1-\alpha = 0.90$.}
  \label{tab:real_world_rf_cov}
\end{table}

\paragraph{Volume.}
\tableref{tab:real_world_rf_vol} reports $\log_{10}$ volume.  $\mathrm{SC}^2$ with
$\gamma=1-\alpha$ achieves the best mean rank across the seven evaluated
methods (mean rank $2.03$) and ranks first on 12 of 29 datasets; the
volume advantage over Max grows consistently with $d$.
Among hyperrectangular methods, our SCO variant achieves a mean
rank strictly better than all baselines (Max, CHR, TSCP), making it the
strongest fully rectangular option. TSCP is competitive at low $d$ but
falls behind at high $d$.  The $\mathrm{SC}^2$ $\gamma_{\mathrm{rect}}$ variant, which corresponds to the valid version
of the copula-based construction (see \sectionref{sec:gamma-choice}), does not improve
consistently over the default $\gamma=1-\alpha$, suggesting it offers no
systematic advantage.

\begin{table}[htbp]
  \centering
  \small
  \setlength{\tabcolsep}{4pt}
  \begin{tabular}{l r r r r r r r r}
    \toprule
    Dataset & $d$ & \multicolumn{1}{c}{\multirow{2}{*}{Max}} & \multicolumn{1}{c}{\multirow{2}{*}{CHR}} & \multicolumn{1}{c}{\multirow{2}{*}{TSCP}} & \multicolumn{1}{c}{SCO} & \multicolumn{3}{c}{$\mathrm{SC}^2$} \\
    \cmidrule(lr){6-6} \cmidrule(lr){7-9}
    & &  &  &  & \multicolumn{1}{c}{$\gamma{=}1{-}\alpha$} & \multicolumn{1}{c}{$\gamma{=}0.8$} & \multicolumn{1}{c}{$\gamma{=}1{-}\alpha$} & \multicolumn{1}{c}{$\gamma_{\mathrm{rect}}$} \\
    \midrule
    ANSUR-II & 2 & $3.40{\scriptstyle \pm 0.01}$ & $\mathbf{3.33}{\scriptstyle \pm 0.01}$ & $3.34{\scriptstyle \pm 0.01}$ & $3.35{\scriptstyle \pm 0.01}$ & $3.34{\scriptstyle \pm 0.01}$ & $3.34{\scriptstyle \pm 0.01}$ & $3.35{\scriptstyle \pm 0.01}$ \\
    Bio & 2 & $6.08{\scriptstyle \pm 0.00}$ & $\mathbf{4.01}{\scriptstyle \pm 0.00}$ & $4.24{\scriptstyle \pm 0.01}$ & $4.01{\scriptstyle \pm 0.00}$ & $4.01{\scriptstyle \pm 0.00}$ & $4.01{\scriptstyle \pm 0.00}$ & $4.01{\scriptstyle \pm 0.00}$ \\
    Births-1 & 2 & $6.42{\scriptstyle \pm 0.00}$ & $4.05{\scriptstyle \pm 0.01}$ & $4.05{\scriptstyle \pm 0.01}$ & $4.04{\scriptstyle \pm 0.00}$ & $4.04{\scriptstyle \pm 0.00}$ & $\mathbf{4.04}{\scriptstyle \pm 0.00}$ & $4.04{\scriptstyle \pm 0.00}$ \\
    Blog & 2 & $3.33{\scriptstyle \pm 0.01}$ & $3.34{\scriptstyle \pm 0.01}$ & $3.57{\scriptstyle \pm 0.01}$ & $3.33{\scriptstyle \pm 0.01}$ & $3.34{\scriptstyle \pm 0.01}$ & $\mathbf{3.33}{\scriptstyle \pm 0.01}$ & $3.42{\scriptstyle \pm 0.01}$ \\
    CalCOFI & 2 & $1.59{\scriptstyle \pm 0.00}$ & $\mathbf{0.94}{\scriptstyle \pm 0.00}$ & $0.94{\scriptstyle \pm 0.00}$ & $0.94{\scriptstyle \pm 0.00}$ & $0.94{\scriptstyle \pm 0.00}$ & $0.94{\scriptstyle \pm 0.00}$ & $0.94{\scriptstyle \pm 0.00}$ \\
    EDM & 2 & $0.47{\scriptstyle \pm 0.03}$ & $0.56{\scriptstyle \pm 0.07}$ & $0.41{\scriptstyle \pm 0.05}$ & $0.56{\scriptstyle \pm 0.07}$ & $0.51{\scriptstyle \pm 0.06}$ & $0.45{\scriptstyle \pm 0.06}$ & $\mathbf{0.38}{\scriptstyle \pm 0.04}$ \\
    ENB & 2 & $1.74{\scriptstyle \pm 0.02}$ & $1.38{\scriptstyle \pm 0.04}$ & $1.32{\scriptstyle \pm 0.02}$ & $1.34{\scriptstyle \pm 0.02}$ & $1.33{\scriptstyle \pm 0.02}$ & $\mathbf{1.32}{\scriptstyle \pm 0.02}$ & $1.36{\scriptstyle \pm 0.02}$ \\
    House & 2 & $11.17{\scriptstyle \pm 0.00}$ & $5.23{\scriptstyle \pm 0.01}$ & $5.27{\scriptstyle \pm 0.01}$ & $5.22{\scriptstyle \pm 0.01}$ & $\mathbf{5.22}{\scriptstyle \pm 0.01}$ & $5.22{\scriptstyle \pm 0.01}$ & $5.24{\scriptstyle \pm 0.01}$ \\
    Taxi & 2 & $-1.64{\scriptstyle \pm 0.00}$ & $-1.63{\scriptstyle \pm 0.01}$ & $-1.63{\scriptstyle \pm 0.00}$ & $-1.64{\scriptstyle \pm 0.00}$ & $\mathbf{-1.64}{\scriptstyle \pm 0.00}$ & $-1.64{\scriptstyle \pm 0.00}$ & $-1.64{\scriptstyle \pm 0.00}$ \\
    Wage & 2 & $\mathbf{0.54}{\scriptstyle \pm 0.00}$ & $0.55{\scriptstyle \pm 0.01}$ & $0.56{\scriptstyle \pm 0.00}$ & $0.55{\scriptstyle \pm 0.00}$ & $0.54{\scriptstyle \pm 0.00}$ & $0.55{\scriptstyle \pm 0.00}$ & $0.56{\scriptstyle \pm 0.00}$ \\
    \midrule
    Jura & 3 & $4.50{\scriptstyle \pm 0.10}$ & $3.67{\scriptstyle \pm 0.11}$ & $3.38{\scriptstyle \pm 0.06}$ & $3.35{\scriptstyle \pm 0.07}$ & $3.33{\scriptstyle \pm 0.08}$ & $\mathbf{3.31}{\scriptstyle \pm 0.07}$ & $3.66{\scriptstyle \pm 0.07}$ \\
    SCPF & 3 & $5.10{\scriptstyle \pm 0.07}$ & $2.88{\scriptstyle \pm 0.08}$ & $2.78{\scriptstyle \pm 0.05}$ & $2.83{\scriptstyle \pm 0.06}$ & $2.94{\scriptstyle \pm 0.07}$ & $\mathbf{2.76}{\scriptstyle \pm 0.06}$ & $2.79{\scriptstyle \pm 0.05}$ \\
    SF1 & 3 & $1.07{\scriptstyle \pm 0.07}$ & $1.90{\scriptstyle \pm 0.27}$ & $\mathbf{1.04}{\scriptstyle \pm 0.08}$ & $1.44{\scriptstyle \pm 0.13}$ & $1.28{\scriptstyle \pm 0.19}$ & $1.11{\scriptstyle \pm 0.12}$ & $1.41{\scriptstyle \pm 0.09}$ \\
    SF2 & 3 & $1.25{\scriptstyle \pm 0.05}$ & $\mathbf{-0.72}{\scriptstyle \pm 0.43}$ & $0.11{\scriptstyle \pm 0.09}$ & $0.51{\scriptstyle \pm 0.04}$ & $0.45{\scriptstyle \pm 0.05}$ & $0.47{\scriptstyle \pm 0.04}$ & $0.03{\scriptstyle \pm 0.15}$ \\
    Student & 3 & $\mathbf{3.46}{\scriptstyle \pm 0.03}$ & $3.48{\scriptstyle \pm 0.05}$ & $3.53{\scriptstyle \pm 0.03}$ & $3.52{\scriptstyle \pm 0.03}$ & $3.49{\scriptstyle \pm 0.03}$ & $3.49{\scriptstyle \pm 0.03}$ & $3.61{\scriptstyle \pm 0.04}$ \\
    \midrule
    Births-2 & 4 & $12.45{\scriptstyle \pm 0.01}$ & $3.14{\scriptstyle \pm 0.04}$ & $3.51{\scriptstyle \pm 0.03}$ & $3.07{\scriptstyle \pm 0.04}$ & $3.99{\scriptstyle \pm 0.05}$ & $3.07{\scriptstyle \pm 0.04}$ & $\mathbf{2.77}{\scriptstyle \pm 0.02}$ \\
    Households & 4 & $19.32{\scriptstyle \pm 0.02}$ & $\mathbf{17.69}{\scriptstyle \pm 0.02}$ & $17.72{\scriptstyle \pm 0.01}$ & $17.71{\scriptstyle \pm 0.01}$ & $17.71{\scriptstyle \pm 0.01}$ & $17.70{\scriptstyle \pm 0.01}$ & $17.72{\scriptstyle \pm 0.01}$ \\
    Stock & 4 & $2.53{\scriptstyle \pm 0.03}$ & $2.42{\scriptstyle \pm 0.04}$ & $2.44{\scriptstyle \pm 0.03}$ & $2.42{\scriptstyle \pm 0.03}$ & $2.41{\scriptstyle \pm 0.03}$ & $\mathbf{2.40}{\scriptstyle \pm 0.03}$ & $2.52{\scriptstyle \pm 0.02}$ \\
    \midrule
    Air & 6 & $9.47{\scriptstyle \pm 0.02}$ & $4.59{\scriptstyle \pm 0.03}$ & $4.65{\scriptstyle \pm 0.01}$ & $4.57{\scriptstyle \pm 0.01}$ & $4.64{\scriptstyle \pm 0.02}$ & $\mathbf{4.57}{\scriptstyle \pm 0.01}$ & $4.65{\scriptstyle \pm 0.02}$ \\
    ATP1d & 6 & $15.88{\scriptstyle \pm 0.17}$ & $15.97{\scriptstyle \pm 0.18}$ & $\mathbf{15.52}{\scriptstyle \pm 0.15}$ & $15.79{\scriptstyle \pm 0.17}$ & $15.65{\scriptstyle \pm 0.15}$ & $15.61{\scriptstyle \pm 0.16}$ & $15.97{\scriptstyle \pm 0.12}$ \\
    ATP7d & 6 & $15.24{\scriptstyle \pm 0.19}$ & $15.85{\scriptstyle \pm 0.25}$ & $15.34{\scriptstyle \pm 0.23}$ & $15.27{\scriptstyle \pm 0.28}$ & $15.14{\scriptstyle \pm 0.26}$ & $\mathbf{15.10}{\scriptstyle \pm 0.27}$ & $15.69{\scriptstyle \pm 0.17}$ \\
    \midrule
    RF1 & 8 & $5.56{\scriptstyle \pm 0.04}$ & $3.54{\scriptstyle \pm 0.07}$ & $4.51{\scriptstyle \pm 0.14}$ & $3.47{\scriptstyle \pm 0.05}$ & $3.62{\scriptstyle \pm 0.04}$ & $\mathbf{3.46}{\scriptstyle \pm 0.05}$ & $3.54{\scriptstyle \pm 0.05}$ \\
    RF2 & 8 & $5.86{\scriptstyle \pm 0.05}$ & $\mathbf{3.94}{\scriptstyle \pm 0.08}$ & $4.87{\scriptstyle \pm 0.16}$ & $3.96{\scriptstyle \pm 0.06}$ & $4.08{\scriptstyle \pm 0.06}$ & $3.95{\scriptstyle \pm 0.06}$ & $4.03{\scriptstyle \pm 0.05}$ \\
    \midrule
    OSales & 12 & $58.14{\scriptstyle \pm 0.40}$ & $52.73{\scriptstyle \pm 0.50}$ & $51.94{\scriptstyle \pm 0.35}$ & $51.90{\scriptstyle \pm 0.37}$ & $52.08{\scriptstyle \pm 0.36}$ & $\mathbf{51.69}{\scriptstyle \pm 0.35}$ & $53.36{\scriptstyle \pm 0.28}$ \\
    \midrule
    WQ & 14 & $12.76{\scriptstyle \pm 0.04}$ & $13.61{\scriptstyle \pm 0.14}$ & $\mathbf{12.47}{\scriptstyle \pm 0.04}$ & $13.34{\scriptstyle \pm 0.07}$ & $13.28{\scriptstyle \pm 0.06}$ & $13.21{\scriptstyle \pm 0.07}$ & $12.59{\scriptstyle \pm 0.04}$ \\
    \midrule
    OES10 & 16 & $56.80{\scriptstyle \pm 0.50}$ & $55.90{\scriptstyle \pm 0.85}$ & $54.25{\scriptstyle \pm 0.42}$ & $53.72{\scriptstyle \pm 0.47}$ & $54.26{\scriptstyle \pm 0.47}$ & $\mathbf{53.41}{\scriptstyle \pm 0.45}$ & $56.44{\scriptstyle \pm 0.63}$ \\
    OES97 & 16 & $59.82{\scriptstyle \pm 0.26}$ & $58.50{\scriptstyle \pm 0.62}$ & $57.48{\scriptstyle \pm 0.42}$ & $56.89{\scriptstyle \pm 0.39}$ & $56.89{\scriptstyle \pm 0.31}$ & $\mathbf{56.51}{\scriptstyle \pm 0.38}$ & $58.66{\scriptstyle \pm 0.51}$ \\
    SCM1d & 16 & $44.28{\scriptstyle \pm 0.06}$ & $43.93{\scriptstyle \pm 0.06}$ & $43.89{\scriptstyle \pm 0.05}$ & $43.88{\scriptstyle \pm 0.05}$ & $43.88{\scriptstyle \pm 0.06}$ & $\mathbf{43.87}{\scriptstyle \pm 0.05}$ & $44.01{\scriptstyle \pm 0.05}$ \\
    SCM20d & 16 & $45.22{\scriptstyle \pm 0.06}$ & $44.93{\scriptstyle \pm 0.07}$ & $44.91{\scriptstyle \pm 0.06}$ & $44.89{\scriptstyle \pm 0.05}$ & $\mathbf{44.86}{\scriptstyle \pm 0.06}$ & $44.88{\scriptstyle \pm 0.05}$ & $45.01{\scriptstyle \pm 0.06}$ \\
    \midrule
    \textbf{Mean rank} & & $5.66$ & $4.24$ & $4.24$ & $3.83$ & $3.31$ & $\mathbf{2.03}$ & $4.69$ \\
    \bottomrule
  \end{tabular}
  \caption{$\log_{10}$ volume on real-world datasets (random-forest base model, $\alpha=0.10$). Mean over 20 seeds. \textbf{Bold}: lowest value per row.}
  \label{tab:real_world_rf_vol}
\end{table}

\paragraph{Staircase compactness.}
\tableref{tab:real_world_frac_retained} reports the fraction of the $2^d$
staircase corner rectangles that survive Pareto pruning
(\algorithmref{alg:staircase}), i.e.\ the proportion of the $\mathrm{SC}^2$
candidate set that is not dominated and therefore stays part of the region.
For the default $\gamma=1-\alpha$, this fraction falls from a mean of $0.57$
at $d=2$ to $0.28$ at $d\in\{3,4\}$, $0.05$ at $d\in\{6,8\}$, and below
$0.001$ at $d\geq 12$: although the number of candidate rectangles grows as
$2^d$, pruning keeps the active set small in practice, so $\mathrm{SC}^2$
remains cheap to evaluate at the dimensions considered here.

The $\gamma_{\mathrm{rect}}$ column illustrates the boundary case discussed in
\sectionref{sec:gamma-choice}: when the base rectangle covers a fraction
$1-\alpha$ of calibration points, $\hatlam_{\bar{S},S}\leq 1$ for every
$S\neq\varnothing$ and the staircase collapses to its single empty-subset
corner, $1/2^d$ of the candidate set. This holds on 25 of the 29 datasets.
It fails on ATP1d, ATP7d, OES10, and OES97, most severely on OES10 and
OES97, where $d=16$ but the calibration set has only $n=50$ to $60$
points, too few for any rectangle, even the loosest one, to jointly cover
$1-\alpha$ of calibration points across all 16 targets. Up to several
thousand of the $65{,}536$ candidate rectangles then survive pruning.

\begin{table}[htbp]
  \centering
  \small
  \setlength{\tabcolsep}{4pt}
  \begin{tabular}{l r r r r}
    \toprule
    Dataset & $d$ & \multicolumn{3}{c}{$\mathrm{SC}^2$} \\
    \cmidrule(lr){3-5}
    & & \multicolumn{1}{c}{$\gamma{=}0.8$} & \multicolumn{1}{c}{$\gamma{=}1{-}\alpha$} & \multicolumn{1}{c}{$\gamma_{\mathrm{rect}}$} \\
    \midrule
    ANSUR-II & 2 & $0.537$ & $0.613$ & $0.250$ \\
    Bio & 2 & $0.562$ & $0.662$ & $0.250$ \\
    Births-1 & 2 & $0.450$ & $0.512$ & $0.250$ \\
    Blog & 2 & $0.537$ & $0.525$ & $0.250$ \\
    CalCOFI & 2 & $0.600$ & $0.613$ & $0.250$ \\
    EDM & 2 & $0.500$ & $0.550$ & $0.250$ \\
    ENB & 2 & $0.613$ & $0.613$ & $0.250$ \\
    House & 2 & $0.550$ & $0.588$ & $0.250$ \\
    Taxi & 2 & $0.575$ & $0.650$ & $0.250$ \\
    Wage & 2 & $0.400$ & $0.362$ & $0.250$ \\
    \midrule
    Jura & 3 & $0.338$ & $0.381$ & $0.125$ \\
    SCPF & 3 & $0.287$ & $0.381$ & $0.125$ \\
    SF1 & 3 & $0.250$ & $0.344$ & $0.125$ \\
    SF2 & 3 & $0.231$ & $0.256$ & $0.125$ \\
    Student & 3 & $0.350$ & $0.338$ & $0.125$ \\
    \midrule
    Births-2 & 4 & $0.078$ & $0.081$ & $0.062$ \\
    Households & 4 & $0.228$ & $0.231$ & $0.062$ \\
    Stock & 4 & $0.212$ & $0.169$ & $0.062$ \\
    \midrule
    Air & 6 & $0.048$ & $0.070$ & $0.016$ \\
    ATP1d & 6 & $0.059$ & $0.080$ & $0.017$ \\
    ATP7d & 6 & $0.052$ & $0.055$ & $0.054$ \\
    \midrule
    RF1 & 8 & $0.011$ & $0.015$ & $0.004$ \\
    RF2 & 8 & $0.011$ & $0.019$ & $0.004$ \\
    \midrule
    OSales & 12 & $0.001$ & $0.001$ & $0.001$ \\
    \midrule
    WQ & 14 & $0.000$ & $0.000$ & $0.000$ \\
    \midrule
    OES10 & 16 & $0.000$ & $0.000$ & $0.004$ \\
    OES97 & 16 & $0.000$ & $0.000$ & $0.007$ \\
    SCM1d & 16 & $0.000$ & $0.000$ & $0.000$ \\
    SCM20d & 16 & $0.000$ & $0.000$ & $0.000$ \\
    \bottomrule
  \end{tabular}
  \caption{Fraction of the $2^d$ staircase corner rectangles that survive Pareto pruning (random-forest base model, $\alpha=0.10$). Mean over 20 seeds. Lower values mean the staircase collapses to fewer active rectangles.}
  \label{tab:real_world_frac_retained}
\end{table}

\paragraph{Runtime.}
\tableref{tab:runtime} reports mean calibration time and formal complexity,
for the methods compared in \tableref{tab:real_world_rf_vol}, on synthetic
calibration errors at $d=8$, $n_\mathrm{calib}=300$ (prediction is negligible
for every method, well under one microsecond per point, and is omitted).
Max and our SCO variant reduce to a handful of order-statistic computations
and take a fraction of a millisecond; CHR has the same $O(nd\log n)$
complexity but a larger constant from its two-stage calibration split. The
staircase variants with fixed $\gamma$ ($\mathrm{SC}^2$, $\gamma \in
\{1-\alpha, 0.8\}$) enumerate up to $2^d$ candidate corner rectangles and
prune the dominated ones with a pairwise comparison, costing under $3$~ms at
$d=8$; the calibration-adaptive $\gamma_{\mathrm{rect}}$ variant additionally
scans $O(n)$ candidate base-rectangle levels, each requiring an $O(nd)$ pass
over the calibration set, which raises its calibration time to about
$7$~ms. TSCP enumerates all partition cells in principle ($O(n^d)$), but
falls back to an $O(dn\log n)$ coordinate-wise search once $d$ exceeds a
small cutoff, which is the regime at $d=8$ and explains why its time is
comparable to the staircase variants despite the worst-case exponent. All
calibration costs remain on the order of milliseconds or less at this
scale, well within the cost of fitting the base regressor.

\begin{table}[htbp]
  \centering
  \small
  \begin{tabular}{l r l}
    \toprule
    Method & Calib.\ time (ms) & Complexity \\
    \midrule
    Max & $0.034$ & $O(nd + n\log n)$ \\
    CHR & $0.096$ & $O(nd\log n)$ \\
    TSCP & $4.654$ & $O(\min(n^d,\, dn\log n))$ \\
    SCO ($\gamma$=1-$\alpha$) & $0.059$ & $O(nd\log n)$ \\
    $\mathrm{SC}^2$ ($\gamma$=0.8) & $2.756$ & $O(2^d n\log n + 4^d)$ \\
    $\mathrm{SC}^2$ ($\gamma$=1-$\alpha$) & $2.770$ & $O(2^d n\log n + 4^d)$ \\
    $\mathrm{SC}^2$ ($\gamma_{\mathrm{rect}}$) & $7.291$ & $O(n^2 d + 2^d n\log n + 4^d)$ \\
    \bottomrule
  \end{tabular}
  \caption{Mean calibration time and formal complexity ($n$ = $n_\mathrm{calib}$, $d$ = number of targets) for the methods compared in \tableref{tab:real_world_rf_vol}. Timings measured at $d=8$, $n_\mathrm{calib}=300$, $\alpha=0.1$, synthetic errors, 10 seeds.}
  \label{tab:runtime}
\end{table}

%% ==========================================================================
\section{Conclusion}
\label{sec:conclusion}
%% ==========================================================================

We introduced the scaling-score method for conformal multi-target regression.
The method wraps any point predictor, requires only component-wise absolute
residuals, and uses a single calibration set to produce four nested prediction
regions: an inner rectangle (SCI, tight but without formal coverage guarantee), the
exact set $\calR_\alpha$ (valid and tightest, characterised cell by cell), the
staircase ($\mathrm{SC}^2$, a valid union of at most $2^d$ hyperrectangles, computable at
calibration time), and the outer rectangle (SCO, a single closed-form valid
hyperrectangle).  All output types share a single calibration run and a single
hyperparameter $\gamma$, for which $\gamma = 1-\alpha$ is a natural and
empirically reliable default.  On the theoretical side, we gave an exact
cell-by-cell characterisation of $\calR_\alpha$, proved downward-closedness,
and derived closed-form sandwich bounds.

Experiments on 29 real-world datasets spanning $d \in \{2,\ldots,16\}$ targets
confirm valid joint coverage throughout.  $\mathrm{SC}^2$ consistently matches
or outperforms all baselines in volume, with the advantage over max-aggregation
growing systematically with $d$, and the closed-form SCO outperforms
all hyperrectangular baselines in mean rank.  Future directions include extending
the method to conditional coverage guarantees, to non-exchangeable settings such
as time series, and to data-driven selection of $\gamma$.

\bibliography{refs}

\begin{thebibliography}{19}
\providecommand{\natexlab}[1]{#1}
\providecommand{\url}[1]{\texttt{#1}}
\expandafter\ifx\csname urlstyle\endcsname\relax
  \providecommand{\doi}[1]{doi: #1}\else
  \providecommand{\doi}{doi: \begingroup \urlstyle{rm}\Url}\fi

\bibitem[Camehl et~al.(2025)Camehl, Fok, and Gruber]{camehl_superlevel_2025}
Annika Camehl, Dennis Fok, and Kathrin Gruber.
\newblock On superlevel sets of conditional densities and multivariate quantile
  regression.
\newblock \emph{Journal of Econometrics}, 249:\penalty0 105807, May 2025.
\newblock ISSN 0304-4076.
\newblock \doi{10.1016/j.jeconom.2024.105807}.

\bibitem[Cevid et~al.(2022)Cevid, Michel, N{\"a}f, B{\"u}hlmann, and
  Meinshausen]{cevid_distributional_2022}
Domagoj Cevid, Loris Michel, Jeffrey N{\"a}f, Peter B{\"u}hlmann, and Nicolai
  Meinshausen.
\newblock Distributional random forests: {{Heterogeneity}} adjustment and
  multivariate distributional regression.
\newblock \emph{Journal of Machine Learning Research}, 23\penalty0
  (333):\penalty0 1--79, 2022.

\bibitem[{del Barrio} et~al.(2022){del Barrio}, Sanz, and
  Hallin]{delbarrio_nonparametric_2022}
Eustasio {del Barrio}, Alberto~Gonzalez Sanz, and Marc Hallin.
\newblock Nonparametric {{Multiple-Output Center-Outward Quantile Regression}},
  April 2022.

\bibitem[Dheur et~al.(2025)Dheur, Fontana, Estievenart, Desobry, and
  Taieb]{dheur_multioutput_2025}
Victor Dheur, Matteo Fontana, Yorick Estievenart, Naomi Desobry, and
  Souhaib~Ben Taieb.
\newblock Multi-{{Output Conformal Regression}}: {{A Unified Comparative
  Study}} with {{New Conformity Scores}}, January 2025.

\bibitem[D{\v z}eroski et~al.(2000)D{\v z}eroski, Dem{\v s}ar, and
  Grbovi{\'c}]{dzeroski_predicting_2000}
Sa{\v s}o D{\v z}eroski, Damjan Dem{\v s}ar, and Jasna Grbovi{\'c}.
\newblock Predicting chemical parameters of river water quality from
  bioindicator data.
\newblock \emph{Applied Intelligence}, 13\penalty0 (1):\penalty0 7--17, 2000.

\bibitem[Fan and Sesia(2025)]{fan_interpretable_2025}
Yunjie Fan and Matteo Sesia.
\newblock Interpretable {{Multivariate Conformal Prediction}} with {{Fast
  Transductive Standardization}}, December 2025.

\bibitem[Feldman et~al.(2023)Feldman, Bates, and
  Romano]{feldman_calibrated_2023}
Shai Feldman, Stephen Bates, and Yaniv Romano.
\newblock Calibrated multiple-output quantile regression with representation
  learning.
\newblock \emph{Journal of Machine Learning Research}, 24\penalty0
  (24):\penalty0 1--48, 2023.

\bibitem[Goovaerts(1997)]{goovaerts_geostatistics_1997}
Pierre Goovaerts.
\newblock \emph{Geostatistics for Natural Resources Evaluation}.
\newblock Oxford university press, 1997.

\bibitem[Izbicki et~al.(2020)Izbicki, Shimizu, and Stern]{izbicki_cdsplit_2020}
Rafael Izbicki, Gilson~T. Shimizu, and R.~Stern.
\newblock {{CD-split}} and {{HPD-split}}: {{Efficient}} conformal regions in
  high dimensions.
\newblock \emph{Journal of Machine Learning Research}, 23:\penalty0
  87:1--87:32, 2020.

\bibitem[Karali{\v c} and Bratko(1997)]{karalic_first_1997}
Aram Karali{\v c} and Ivan Bratko.
\newblock First {{Order Regression}}.
\newblock \emph{Machine Learning}, 26\penalty0 (2-3):\penalty0 147--176,
  February 1997.
\newblock ISSN 0885-6125, 1573-0565.
\newblock \doi{10.1023/A:1007365207130}.

\bibitem[Lei et~al.(2018)Lei, G'Sell, Rinaldo, Tibshirani, and
  Wasserman]{lei_distributionfree_2018}
Jing Lei, Max G'Sell, Alessandro Rinaldo, Ryan~J. Tibshirani, and Larry
  Wasserman.
\newblock Distribution-{{Free Predictive Inference}} for {{Regression}}.
\newblock \emph{Journal of the American Statistical Association}, 113\penalty0
  (523):\penalty0 1094--1111, July 2018.
\newblock ISSN 0162-1459.
\newblock \doi{10.1080/01621459.2017.1307116}.

\bibitem[Messoudi et~al.(2021)Messoudi, Destercke, and
  Rousseau]{messoudi_copulabased_2021}
Soundouss Messoudi, S{\'e}bastien Destercke, and Sylvain Rousseau.
\newblock Copula-based conformal prediction for {{Multi-Target Regression}}.
\newblock \emph{arXiv:2101.12002 [cs, stat]}, January 2021.

\bibitem[Papadopoulos et~al.(2002)Papadopoulos, Proedrou, Vovk, and
  Gammerman]{papadopoulos_inductive_2002}
Harris Papadopoulos, Kostas Proedrou, Volodya Vovk, and Alex Gammerman.
\newblock Inductive confidence machines for regression.
\newblock In \emph{European {{Conference}} on {{Machine Learning}}}, pages
  345--356. Springer, 2002.

\bibitem[Romano et~al.(2019)Romano, Patterson, and
  Cand{\`e}s]{romano_conformalized_2019}
Yaniv Romano, Evan Patterson, and Emmanuel~J. Cand{\`e}s.
\newblock Conformalized {{Quantile Regression}}.
\newblock In Hanna~M. Wallach, Hugo Larochelle, Alina Beygelzimer, Florence
  {d'Alch{\'e}-Buc}, Emily~B. Fox, and Roman Garnett, editors, \emph{Advances
  in {{Neural Information Processing Systems}} 32: {{Annual Conference}} on
  {{Neural Information Processing Systems}} 2019, {{NeurIPS}} 2019,
  {{December}} 8-14, 2019, {{Vancouver}}, {{BC}}, {{Canada}}}, pages
  3538--3548, 2019.

\bibitem[Sampson and Chan(2024)]{sampson_conformal_2024}
Max Sampson and Kung-Sik Chan.
\newblock Conformal {{Multi}}-{{Target Hyperrectangles}}.
\newblock \emph{Statistical Analysis and Data Mining: The ASA Data Science
  Journal}, 17\penalty0 (5):\penalty0 e11710, October 2024.
\newblock ISSN 1932-1864, 1932-1872.
\newblock \doi{10.1002/sam.11710}.

\bibitem[Tsanas and Xifara(2012)]{tsanas_accurate_2012}
Athanasios Tsanas and Angeliki Xifara.
\newblock Accurate quantitative estimation of energy performance of residential
  buildings using statistical machine learning tools.
\newblock \emph{Energy and buildings}, 49:\penalty0 560--567, 2012.

\bibitem[Tsoumakas et~al.(2011)Tsoumakas, {Spyromitros-Xioufis}, Vilcek, and
  Vlahavas]{tsoumakas_mulan_2011}
Grigorios Tsoumakas, Eleftherios {Spyromitros-Xioufis}, Jozef Vilcek, and
  Ioannis Vlahavas.
\newblock Mulan: {{A}} java library for multi-label learning.
\newblock \emph{The Journal of Machine Learning Research}, 12:\penalty0
  2411--2414, 2011.

\bibitem[Vovk et~al.(2005)Vovk, Gammerman, and Shafer]{vovk_algorithmic_2005}
Vladimir Vovk, Alex Gammerman, and Glenn Shafer.
\newblock \emph{Algorithmic Learning in a Random World}.
\newblock Springer Science \& Business Media, 2005.

\bibitem[Wang et~al.(2022)Wang, Gao, Yin, Zhou, and
  Blei]{wang_probabilistic_2022}
Zhendong Wang, Ruijiang Gao, Mingzhang Yin, Mingyuan Zhou, and David~M. Blei.
\newblock Probabilistic {{Conformal Prediction Using Conditional Random
  Samples}}, June 2022.

\end{thebibliography}

%% ==========================================================================
\clearpage
\appendix
%% ==========================================================================

\section{Proofs}
\label{app:proofs}

\phantomsection\label{proof:validity}%
\begin{proof}[Proof of \theoremref{prop:validity}]
Both $b_{n+1}$ and the scores $\sigma_i(e_{n+1})$ depend only on the multiset
$\{e_1,\ldots,e_n,e_{n+1}\}$ and are symmetric in the indices $\{1,\ldots,n+1\}$.
Under exchangeability, the rank of $\sigma_{n+1}$ among
$\{\sigma_1,\ldots,\sigma_{n+1}\}$ is uniformly distributed on $\{1,\ldots,n+1\}$.
Hence $\mathbb{P}(\sigma_{n+1}(e_{n+1}) \leq \hatlam(e_{n+1})) \geq 1-\alpha$,
which is exactly $\mathbb{P}(e_{n+1} \in \calR_\alpha) \geq 1-\alpha$.
\end{proof}

\phantomsection\label{proof:scores}%
\begin{proof}[Proof of \lemmaref{lem:scores}]
Substitute the cell formula for $b^{(k)}_{n+1}$ into~\eqref{eq:sigma}:
each coordinate contributes $e_{n+1}^{(k)}/L_k$, $e_{n+1}^{(k)}/e_{n+1}^{(k)}=1$,
or $e_{n+1}^{(k)}/U_k$ to $\sigma_{n+1}$, and
$e_i^{(k)}/L_k$, $e_i^{(k)}/e_{n+1}^{(k)}$, or $e_i^{(k)}/U_k$ to $\sigma_i$.
Non-increase of $e_i^{(k)}/e_{n+1}^{(k)}$ in $e_{n+1}^{(k)}$ is immediate.
\end{proof}

\phantomsection\label{proof:corner-exact}%
\begin{proof}[Proof of \propositionref{prop:corner-exact}]
The intersection requires, for each coordinate:
\begin{itemize}
\item $k \in \Ilow$: $e_{n+1}^{(k)} \leq L_k$ (cell) and
  $e_{n+1}^{(k)} \leq \hatlam L_k$ (rectangle), giving
  $e_{n+1}^{(k)} \leq \min(\hatlam,1) L_k$.
\item $k \in \Ihigh$: $e_{n+1}^{(k)} > U_k$ (cell) and
  $e_{n+1}^{(k)} \leq \hatlam U_k$ (rectangle), giving
  $e_{n+1}^{(k)} \in (U_k,\,\hatlam U_k]$, non-empty iff $\hatlam > 1$.
\end{itemize}
\end{proof}

\phantomsection\label{proof:mixed-monotone}%
\begin{proof}[Proof of \propositionref{prop:mixed-monotone}]
Each $\sigma_i(e_{n+1}) = \max(D_i,\,
  \max_{k \in \Imid} e_i^{(k)}/e_{n+1}^{(k)})$,
where $D_i = \max(\max_{k \in \Ilow} e_i^{(k)}/L_k,\,
  \max_{k \in \Ihigh} e_i^{(k)}/U_k)$ is constant on the cell.
Since $e_i^{(k)}/e_{n+1}^{(k)}$ is decreasing in $e_{n+1}^{(k)}$,
$\sigma_i$ is non-increasing in each mid coordinate,
and so is the $q$-th order statistic.
The boundary values follow by substituting $e_{n+1}^{(k)} = L_k$
(resp.\ $U_k$) in the mid terms, which gives
$e_i^{(k)}/L_k$ (resp.\ $e_i^{(k)}/U_k$):
the scores then match those of the corner cell
$(\Ilow \cup \Imid,\,\varnothing,\,\Ihigh)$
(resp.\ $(\Ilow,\,\varnothing,\,\Imid \cup \Ihigh)$).
Since $\hatlam$ is the quantile of these scores, it too matches the
adjacent corner-cell value at each face, so the boundary of $\calR_\alpha$
is continuous across cell boundaries.
\end{proof}

\phantomsection\label{proof:mixed-exact}%
\begin{proof}[Proof of \propositionref{prop:mixed-exact}]
The test score formula follows from \lemmaref{lem:scores}: on the cell, the low
terms $e_{n+1}^{(k)}/L_k \leq 1$ and the mid terms equal $1$, so
$\sigma_{n+1} = \max(1,\,\max_{k \in \Ihigh} e_{n+1}^{(k)}/U_k)$.

For item~1: $\sigma_{n+1} = 1$, so $\sigma_{n+1} \leq \hatlam$ iff
$\hatlam \geq 1$.  The low coordinates satisfy
$e_{n+1}^{(k)} \leq L_k$ automatically.

For item~2: $\sigma_{n+1} \leq \hatlam$ iff
$e_{n+1}^{(k)} \leq \hatlam\,U_k$ for all $k \in \Ihigh$.
Combined with $e_{n+1}^{(k)} > U_k$, this requires $\hatlam > 1$.
\end{proof}

\phantomsection\label{proof:downward-closed}%
\begin{proof}[Proof of \propositionref{prop:downward-closed}]
\emph{Base vector.}
From~\eqref{eq:bk-formula},
$b^{(k)}(e) = \operatorname{clamp}(e^{(k)},L_k,U_k)$
is non-decreasing in $e^{(k)}$, so
$e' \preceq e_{n+1} \Rightarrow b(e') \preceq b(e_{n+1})$.

\emph{Test score.}
Since $x/\operatorname{clamp}(x,L,U)$ is non-decreasing (each branch $x/L$, $1$,
$x/U$ is non-decreasing), for each $k$:
$e'^{(k)}/b^{(k)}(e') \leq e_{n+1}^{(k)}/b^{(k)}(e_{n+1})$.
Taking the max: $\sigma_{n+1}(e') \leq \sigma_{n+1}(e_{n+1})$.

\emph{Calibration quantile.}
Since $b(e') \preceq b(e_{n+1})$, each calibration score satisfies
$\sigma_i(e') \geq \sigma_i(e_{n+1})$ for all $i \leq n$,
so $\hatlam(b(e')) \geq \hatlam(b(e_{n+1}))$.

\emph{Conclusion.}
$\sigma_{n+1}(e') \leq \sigma_{n+1}(e_{n+1}) \leq \hatlam(b(e_{n+1}))
\leq \hatlam(b(e'))$,
where the middle inequality uses $e_{n+1} \in \calR_\alpha$.
Hence $e' \in \calR_\alpha$.
\end{proof}

\phantomsection\label{proof:outer-bound}%
\begin{proof}[Proof of \propositionref{prop:outer-bound}]
Since $b^{(k)}_{n+1} \geq L_k$, for each $i \leq n$:
$\sigma_i(e_{n+1}) = \max_k e_i^{(k)}/b^{(k)}_{n+1}
  \leq \max_k e_i^{(k)}/L_k = \sigma_i^{\mathrm{ub}}$.
Hence $\hatlam(b_{n+1}) \leq \hatlam_{\mathrm{ub}}$.
For any $e_{n+1} \in \calR_\alpha$, by definition $\sigma_{n+1}(e_{n+1}) \leq \hatlam(e_{n+1})$, so:
$\sigma_{n+1}(e_{n+1}) \leq \hatlam(b_{n+1}) \leq \hatlam_{\mathrm{ub}}$.
By equation~\eqref{eq:containment}: $e_{n+1}^{(k)} \leq \hatlam_{\mathrm{ub}} b^{(k)}_{n+1}
\leq \hatlam_{\mathrm{ub}} U_k$ for all $k$.
\end{proof}

\phantomsection\label{proof:sandwich}%
\begin{proof}[Proof of \propositionref{prop:sandwich}]
\emph{Inner $\subseteq$ exact.}
Let $e_{n+1}^{(k)} \leq \hatlam_{\mathrm{lb}} L_k$ for all $k$.
Since $b^{(k)}_{n+1} \geq L_k$:
\[
  \sigma_{n+1}(e_{n+1})
  = \max_k \frac{e_{n+1}^{(k)}}{b^{(k)}_{n+1}}
  \leq \max_k \frac{e_{n+1}^{(k)}}{L_k}
  \leq \hatlam_{\mathrm{lb}}.
\]
Since $b^{(k)}_{n+1} \leq U_k$, for each $i \leq n$:
$\sigma_i(e_{n+1}) \geq \max_k e_i^{(k)}/U_k = \sigma_i^{\mathrm{lb}}$.
Taking the $q$-th order statistic:
$\hatlam(b_{n+1}) \geq \hatlam_{\mathrm{lb}}$.
Combining: $\sigma_{n+1}(e_{n+1}) \leq \hatlam_{\mathrm{lb}} \leq \hatlam(b_{n+1})$,
i.e.\ $e_{n+1} \in \calR_\alpha$.

\emph{Exact $\subseteq$ staircase.}
This is \propositionref{prop:extremal}.

This proves~\eqref{eq:sandwich}.

\emph{Staircase $\subseteq$ outer rectangle (under $\hatlam_{\mathrm{ub}} \geq 1$).}
Each corner rectangle satisfies $e^{*(k)}_S \leq \hatlam_{\mathrm{ub}} U_k$:
for $k \in S$, $\hatlam_{\bar{S},S} \leq \hatlam_{\mathrm{ub}}$ since
$\sigma_i^{\mathrm{ub}} \geq \sigma_i^{\bar{S},S}$; for $k \notin S$,
$U_k \leq \hatlam_{\mathrm{ub}} U_k$ because $\hatlam_{\mathrm{ub}} \geq 1$.
Hence the union is contained in $\prod_k [0, \hatlam_{\mathrm{ub}} U_k]$,
yielding~\eqref{eq:sandwich-full} together with \propositionref{prop:outer-bound}.
The hypothesis $\hatlam_{\mathrm{ub}} \geq 1$ is necessary: if
$\hatlam_{\mathrm{ub}} < 1$, then $\hatlam_{\mathrm{ub}} U_k < U_k = e^{*(k)}_\varnothing$
in every coordinate, so the empty-subset corner $e^*_\varnothing = U$ of the
staircase already lies strictly outside the outer rectangle.
\end{proof}

\phantomsection\label{proof:extremal}%
\begin{proof}[Proof of \propositionref{prop:extremal}]
Let $e \in \calR_\alpha$ and set $S = \{k : e^{(k)} > U_k\}$.
For $k \notin S$: $e^{(k)} \leq U_k = e^{*(k)}_S$ by definition.
For $k \in S$: $e$ belongs to a cell with $\Ihigh = S$
and $\Ilow \cup \Imid = \bar{S}$.
When $\Imid = \varnothing$, \propositionref{prop:corner-exact} gives
$e^{(k)} \leq \hatlam_{\bar{S},S} U_k$.
When $\Imid \neq \varnothing$, \propositionref{prop:mixed-monotone} gives
$\hatlam(e^{(\Imid)}) \leq \hatlam_{\bar{S},S}$;
the exact-region formula for mixed cells then gives
$e^{(k)} \leq \hatlam(e^{(\Imid)}) U_k \leq \hatlam_{\bar{S},S} U_k$.
In either case $e^{(k)} \leq \hatlam_{\bar{S},S} U_k = e^{*(k)}_S$.
Hence $e \preceq e^*_S$ and $e$ lies in the box indexed by $S$.
\end{proof}

%% ==========================================================================
\section{Dataset Details}
\label{app:datasets}
%% ==========================================================================

\begin{longtable}{l r r r r r l}
  \caption{Real-world datasets used in the experiments.  Split: 60\,\% train / 15\,\% calibration / 25\,\% test  averaged over 20 random seeds.}  \label{tab:datasets}\\
  \toprule
  Dataset & $d$ & $n$ & $n_{\text{train}}$ & $n_{\text{calib}}$ & $n_{\text{test}}$ & Source \\
  \midrule
  \endfirsthead
  \toprule
  Dataset & $d$ & $n$ & $n_{\text{train}}$ & $n_{\text{calib}}$ & $n_{\text{test}}$ & Source \\
  \midrule
  \endhead
  \midrule
  \multicolumn{7}{r}{\small\itshape (continued on next page)}\\
  \endfoot
  \bottomrule
  \endlastfoot
  EDM & 2 & 154 & 92 & 23 & 39 & \citet{karalic_first_1997} \\
  ENB & 2 & 768 & 460 & 115 & 193 & \citet{tsanas_accurate_2012} \\
  ANSUR~II & 2 & 1,986 & 1,191 & 297 & 498 & \citet{delbarrio_nonparametric_2022} \\
  Births~1 & 2 & 10,000 & 6,000 & 1,500 & 2,500 & \citet{cevid_distributional_2022} \\
  Wage & 2 & 10,000 & 6,000 & 1,500 & 2,500 & \citet{cevid_distributional_2022} \\
  Blog Feedback & 2 & 20,000 & 12,000 & 3,000 & 5,000 & \citet{feldman_calibrated_2023} \\
  House Prices & 2 & 21,613 & 12,967 & 3,241 & 5,405 & \citet{feldman_calibrated_2023} \\
  NYC Taxi & 2 & 30,000 & 18,000 & 4,500 & 7,500 & \citet{wang_probabilistic_2022} \\
  Bio & 2 & 45,730 & 27,438 & 6,859 & 11,433 & \citet{feldman_calibrated_2023} \\
  CalCOFI & 2 & 50,000 & 30,000 & 7,500 & 12,500 & \citet{delbarrio_nonparametric_2022} \\
  \midrule
  Solar Flare~1 & 3 & 323 & 193 & 48 & 82 & UCI ML Repository \\
  Jura & 3 & 359 & 215 & 53 & 91 & \citet{goovaerts_geostatistics_1997} \\
  Student Perf. & 3 & 395 & 237 & 59 & 99 & UCI ML Repository \\
  Solar Flare~2 & 3 & 1,066 & 639 & 159 & 268 & UCI ML Repository \\
  SCPF & 3 & 1,137 & 682 & 170 & 285 & Kaggle \\
  \midrule
  Stock & 4 & 950 & 570 & 142 & 238 & \citet{wang_probabilistic_2022} \\
  Households & 4 & 7,207 & 4,324 & 1,081 & 1,802 & \citet{camehl_superlevel_2025} \\
  Births~2 & 4 & 10,000 & 6,000 & 1,500 & 2,500 & \citet{cevid_distributional_2022} \\
  \midrule
  ATP~7D & 6 & 296 & 177 & 44 & 75 & \citet{tsoumakas_mulan_2011} \\
  ATP~1D & 6 & 337 & 202 & 50 & 85 & \citet{tsoumakas_mulan_2011} \\
  Air Quality & 6 & 10,000 & 6,000 & 1,500 & 2,500 & \citet{cevid_distributional_2022} \\
  \midrule
  River Flow~2 & 8 & 7,679 & 4,607 & 1,151 & 1,921 & \citet{tsoumakas_mulan_2011} \\
  River Flow~1 & 8 & 9,005 & 5,403 & 1,350 & 2,252 & \citet{tsoumakas_mulan_2011} \\
  \midrule
  Online Sales & 12 & 556 & 333 & 83 & 140 & Kaggle \\
  \midrule
  Water Quality & 14 & 1,060 & 636 & 159 & 265 & \citet{dzeroski_predicting_2000} \\
  \midrule
  OES~1997 & 16 & 334 & 200 & 50 & 84 & \citet{tsoumakas_mulan_2011} \\
  OES~2010 & 16 & 403 & 241 & 60 & 102 & \citet{tsoumakas_mulan_2011} \\
  SCM 20-day & 16 & 8,966 & 5,379 & 1,344 & 2,243 & \citet{tsoumakas_mulan_2011} \\
  SCM 1-day & 16 & 9,803 & 5,881 & 1,470 & 2,452 & \citet{tsoumakas_mulan_2011} \\
\end{longtable}

%% ==========================================================================
\section{Full Experimental Results}
\label{app:full-results}
%% ==========================================================================

Each figure row below corresponds to a group of datasets sharing the same number
of targets~$d$.  The title of each panel follows the format \emph{Dataset name
($d$ targets, $n_{\mathrm{train}}$ train, $n_{\mathrm{calib}}$ calib, $n_{\mathrm{test}}$ test)}.  The left
plot shows joint coverage (dashed line: nominal $1-\alpha$) and the right plot
shows $\log_{10}$ volume, both as a function of the method.  Four groups of
figures are included: random forest with $\alpha=0.10$, random forest with
$\alpha=0.05$, linear model with $\alpha=0.10$, and linear model with
$\alpha=0.05$.  Within each panel, methods are listed in decreasing order of
performance, so the method at the top is the best performer for that particular
dataset.

\subsection*{Random forest, $\alpha=0.10$}
\noindent
\includegraphics[width=\textwidth]{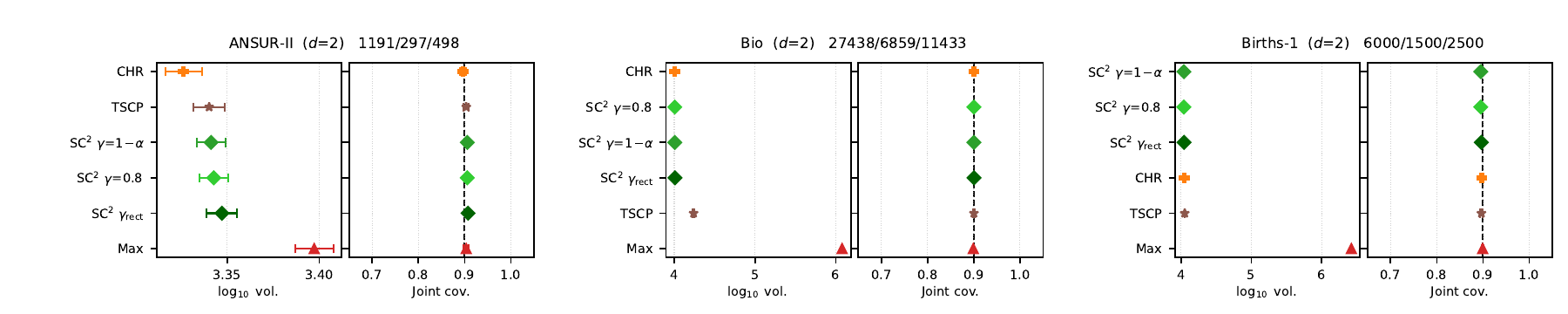}
\includegraphics[width=\textwidth]{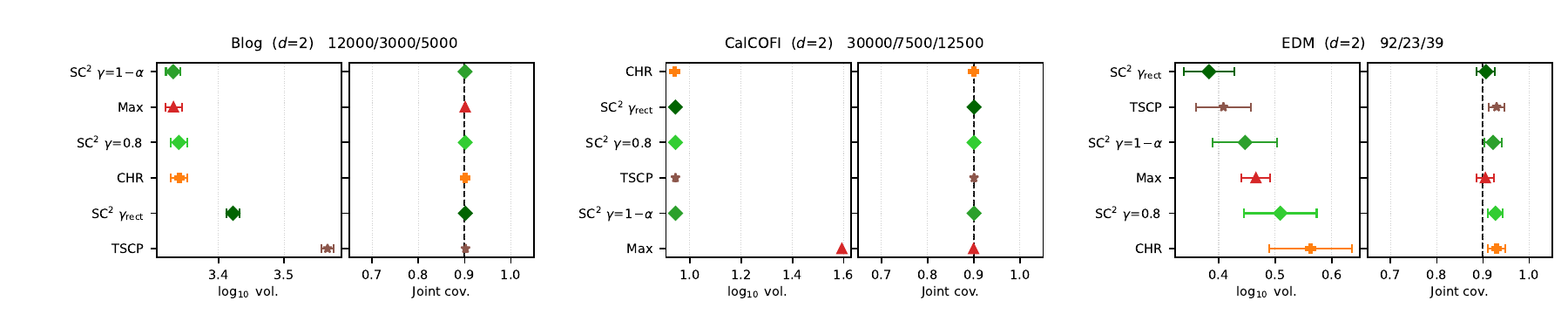}
\includegraphics[width=\textwidth]{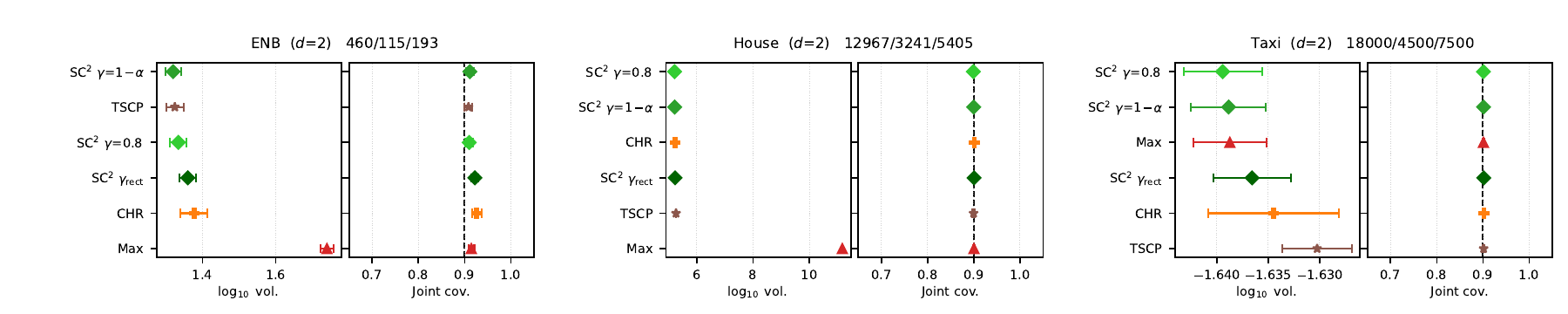}
\includegraphics[width=\textwidth]{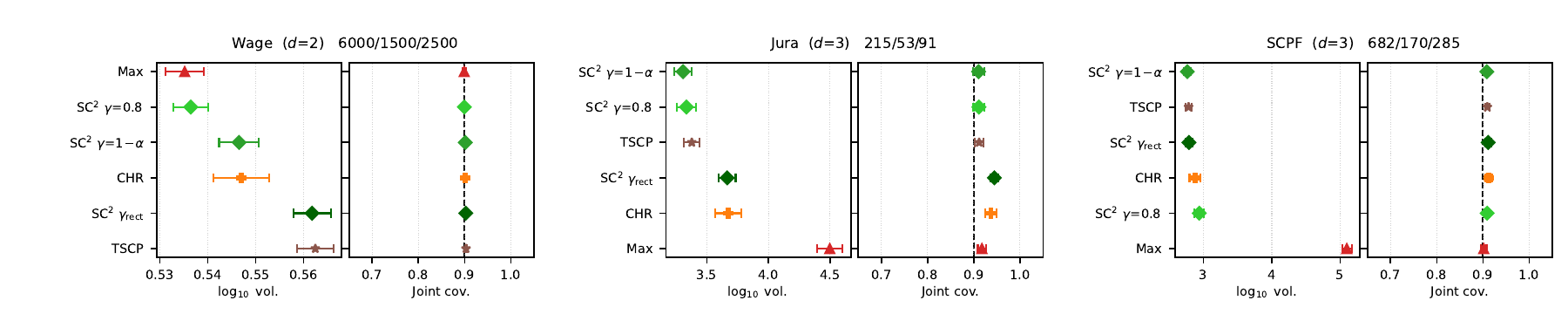}
\includegraphics[width=\textwidth]{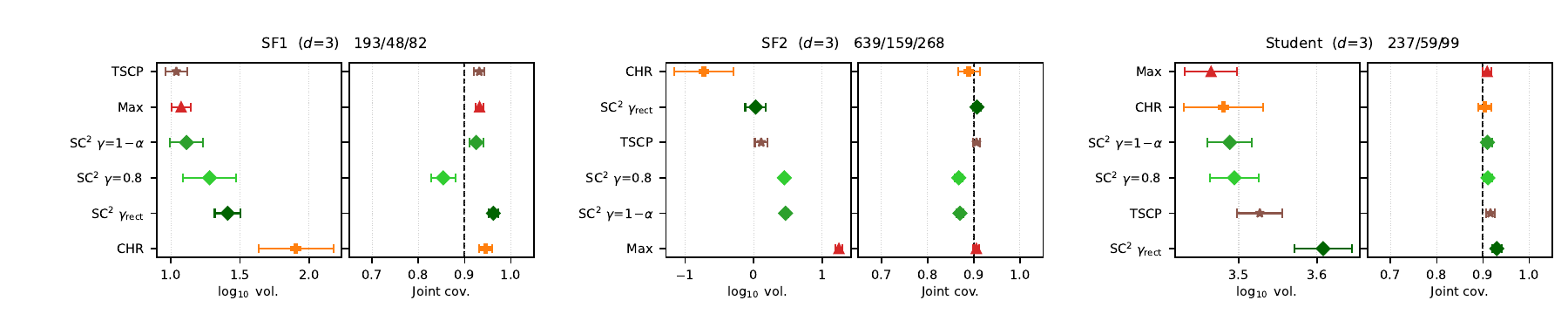}
\includegraphics[width=\textwidth]{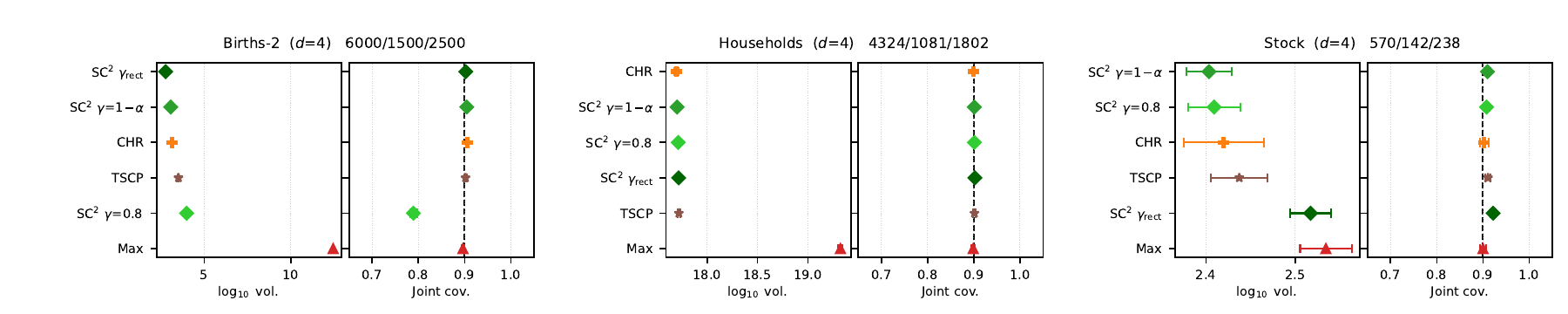}
\includegraphics[width=\textwidth]{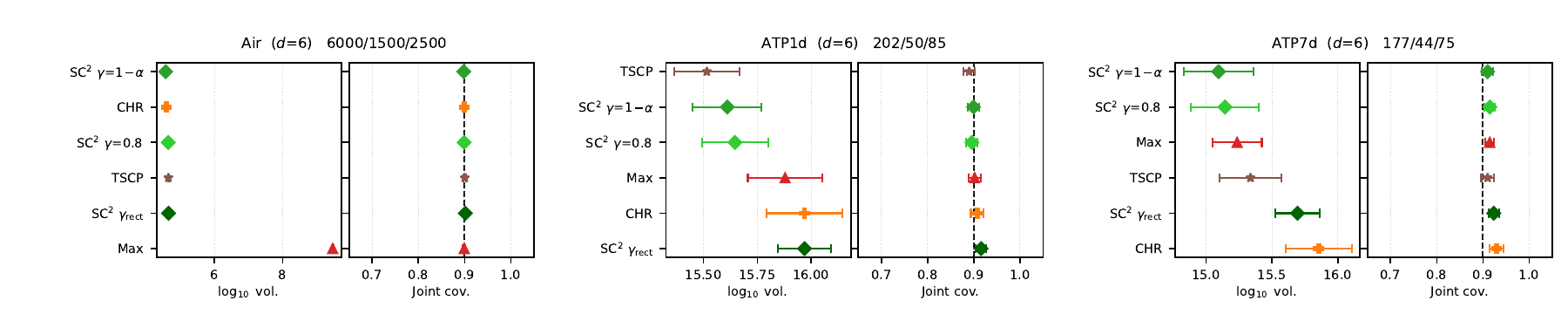}
\includegraphics[width=\textwidth]{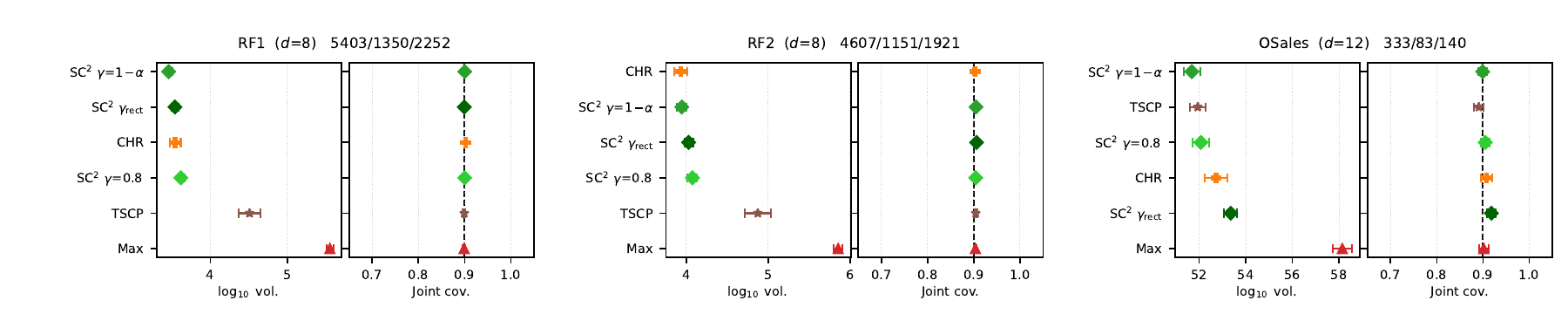}
\includegraphics[width=\textwidth]{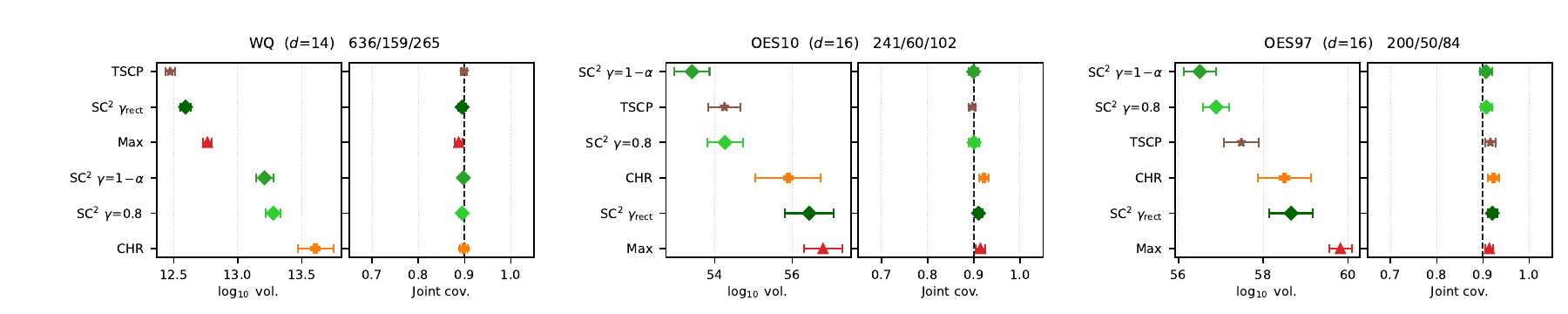}
\includegraphics[width=\textwidth]{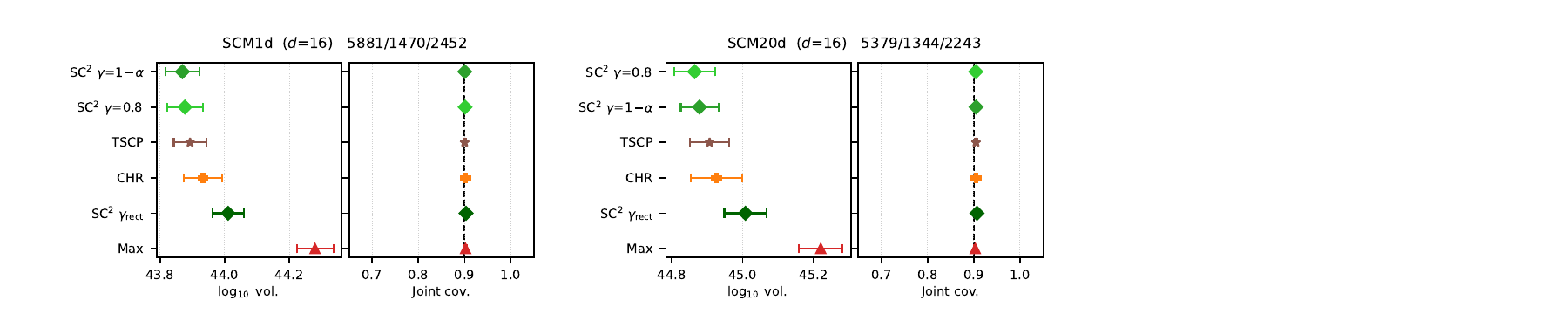}

\subsection*{Random forest, $\alpha=0.05$}
\noindent
\includegraphics[width=\textwidth]{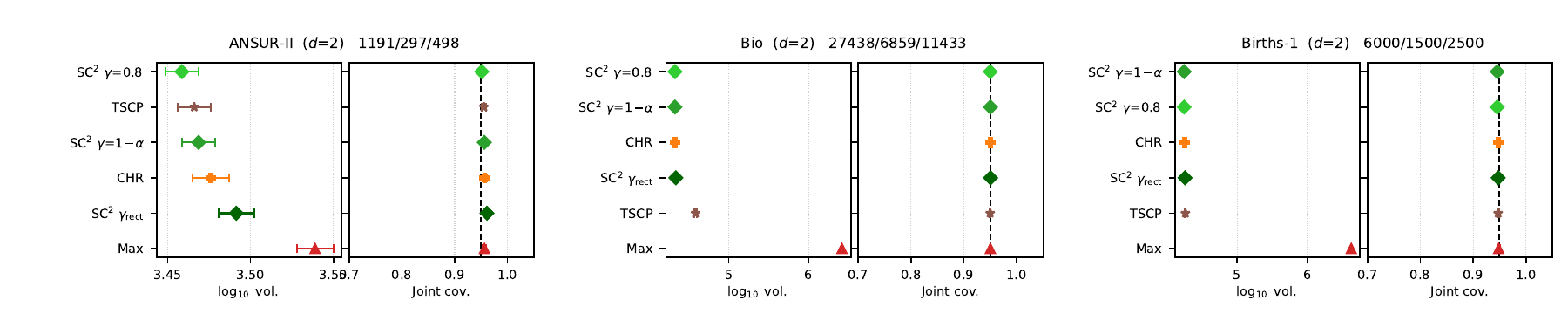}
\includegraphics[width=\textwidth]{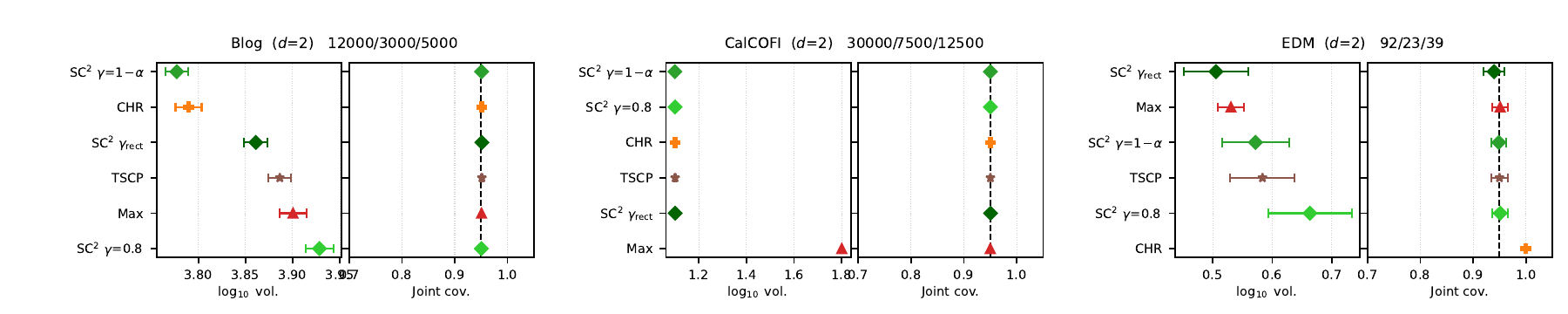}
\includegraphics[width=\textwidth]{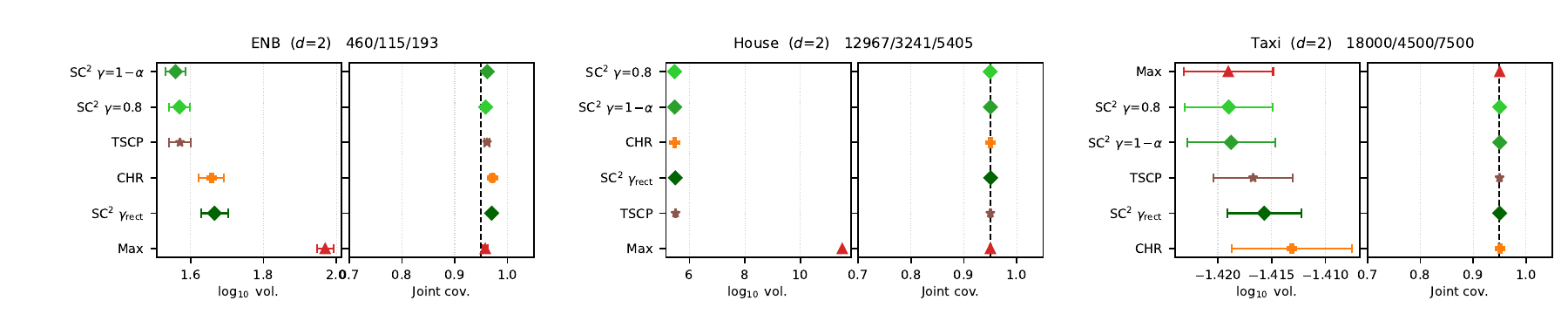}
\includegraphics[width=\textwidth]{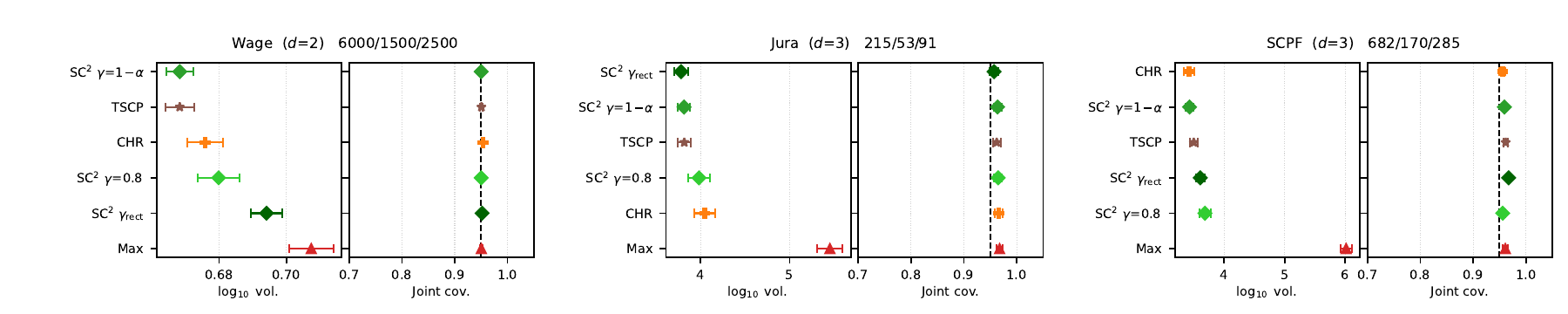}
\includegraphics[width=\textwidth]{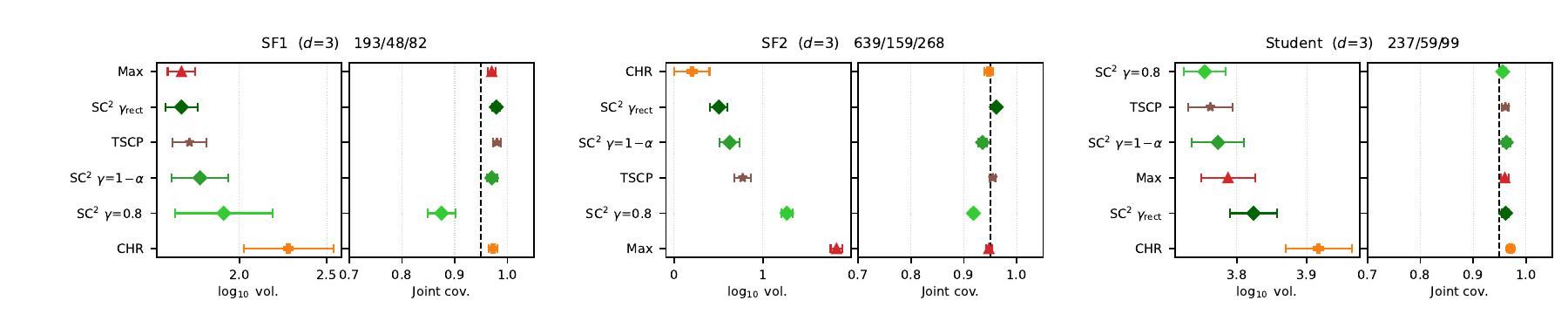}
\includegraphics[width=\textwidth]{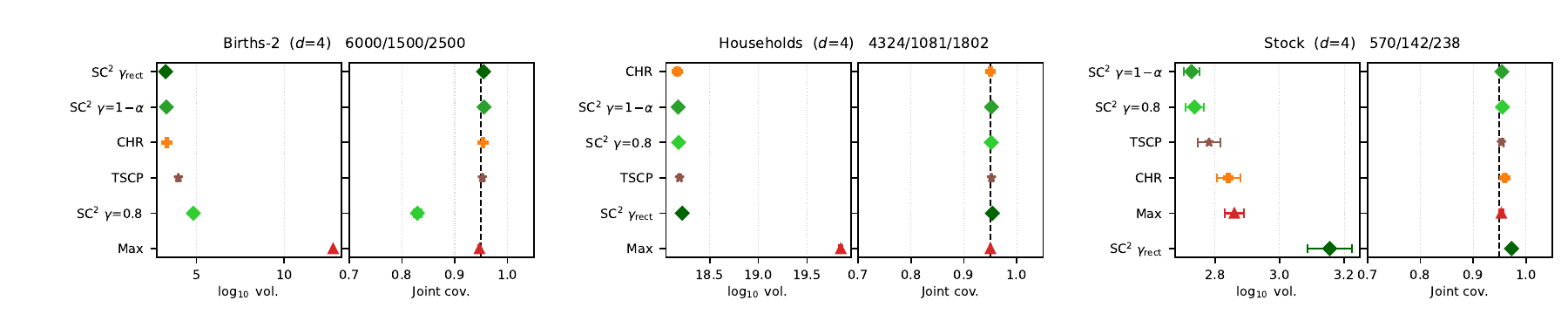}
\includegraphics[width=\textwidth]{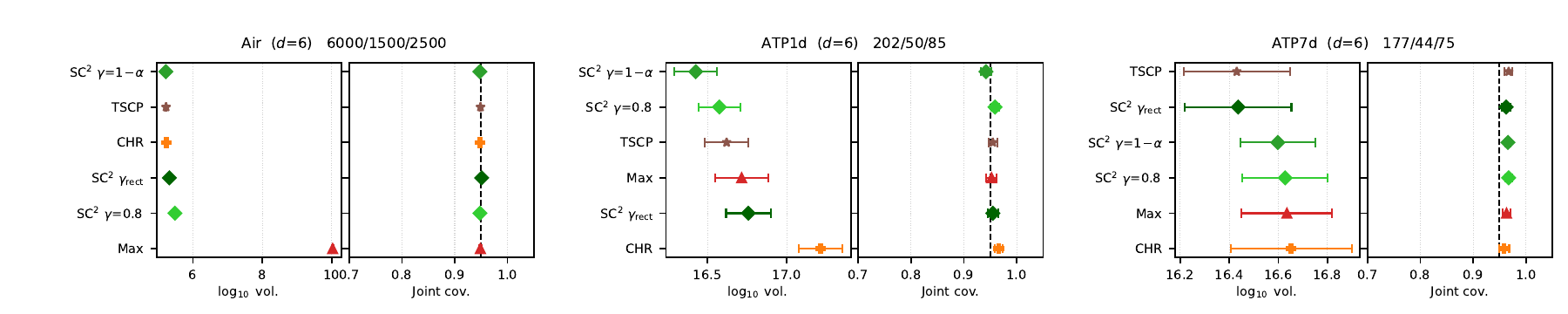}
\includegraphics[width=\textwidth]{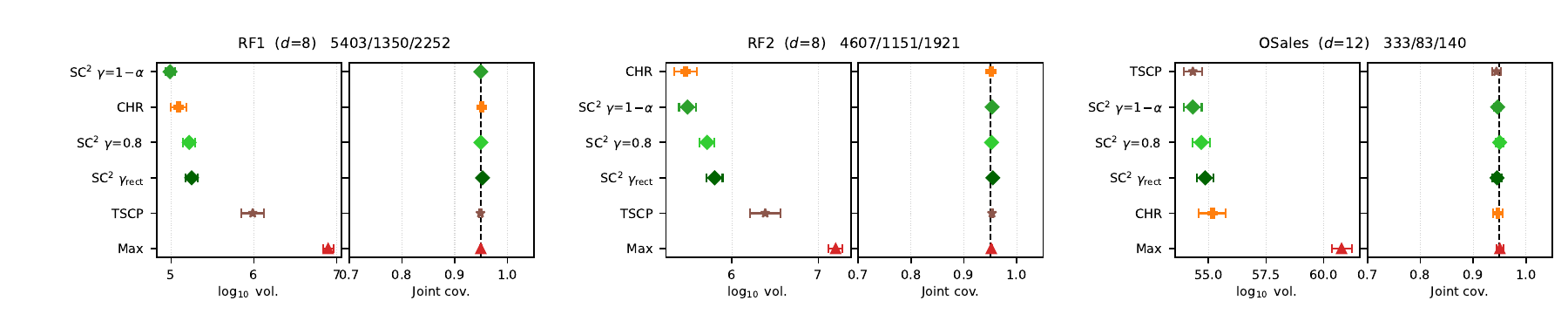}
\includegraphics[width=\textwidth]{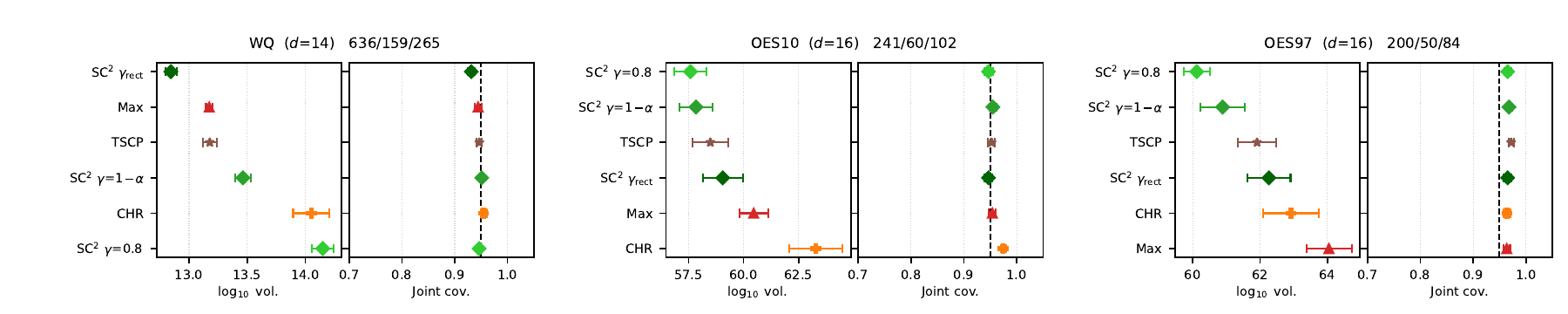}
\includegraphics[width=\textwidth]{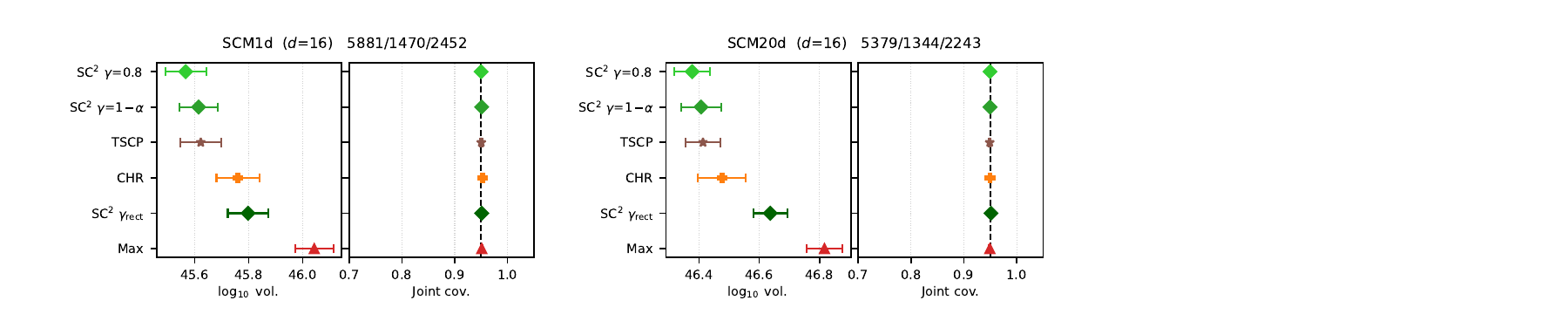}

\subsection*{Linear model, $\alpha=0.10$}
\noindent
\includegraphics[width=\textwidth]{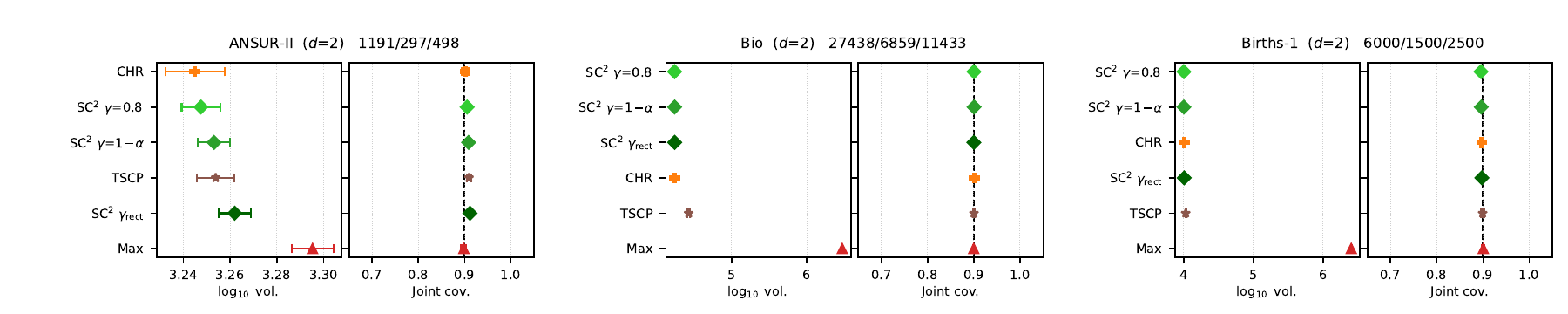}
\includegraphics[width=\textwidth]{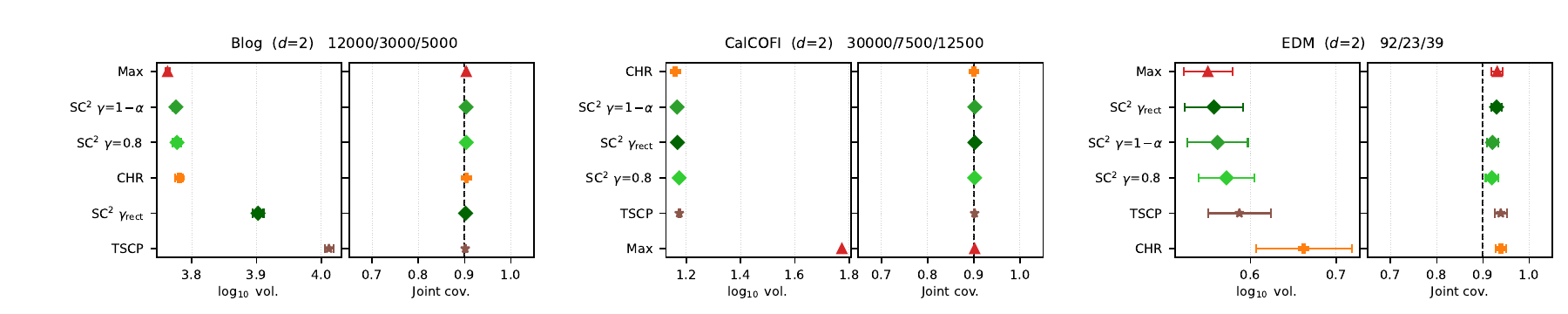}
\includegraphics[width=\textwidth]{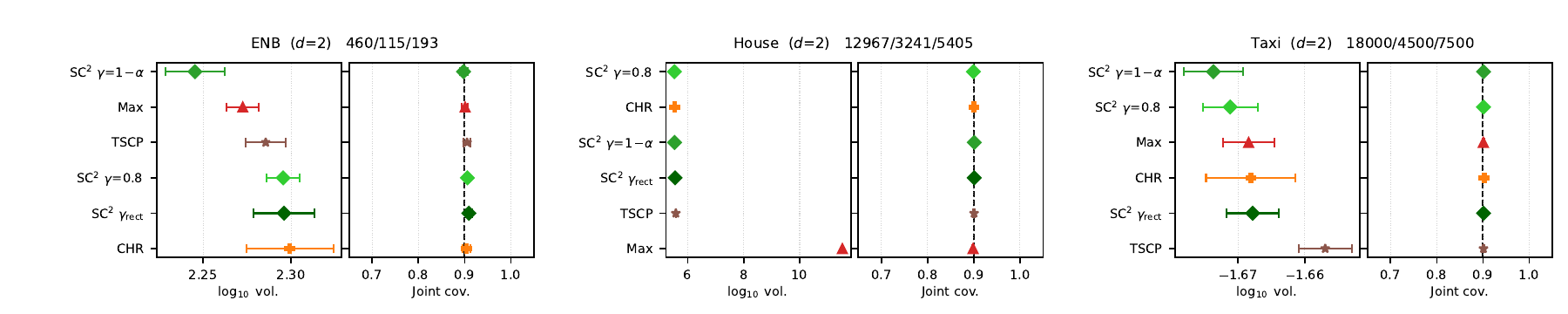}
\includegraphics[width=\textwidth]{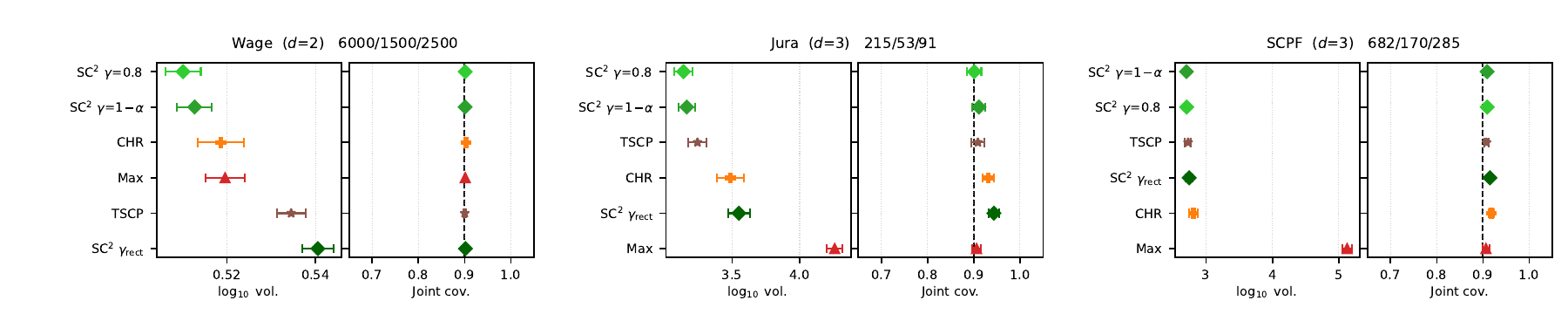}
\includegraphics[width=\textwidth]{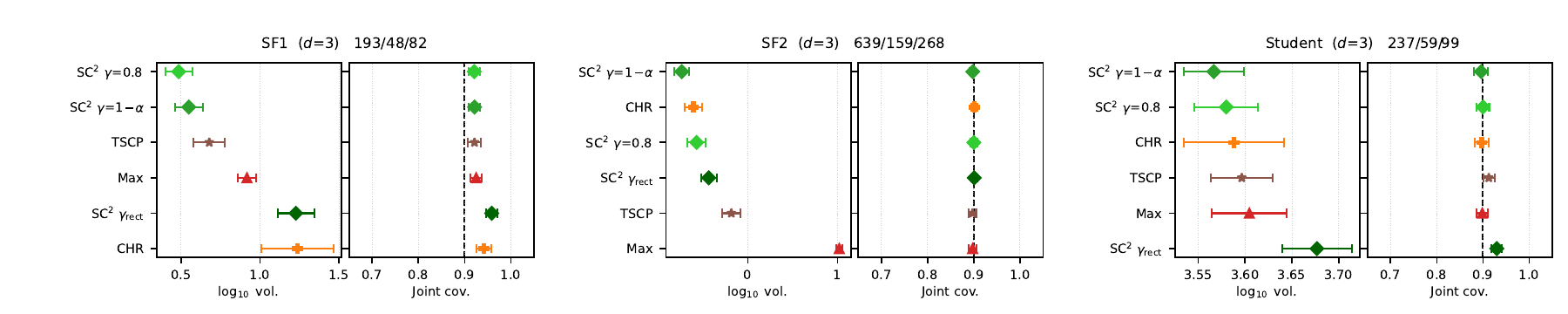}
\includegraphics[width=\textwidth]{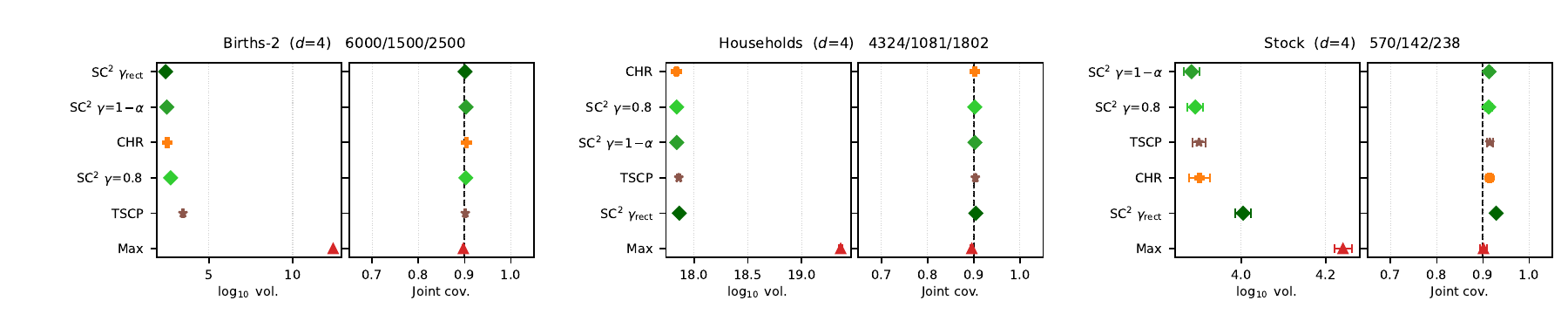}
\includegraphics[width=\textwidth]{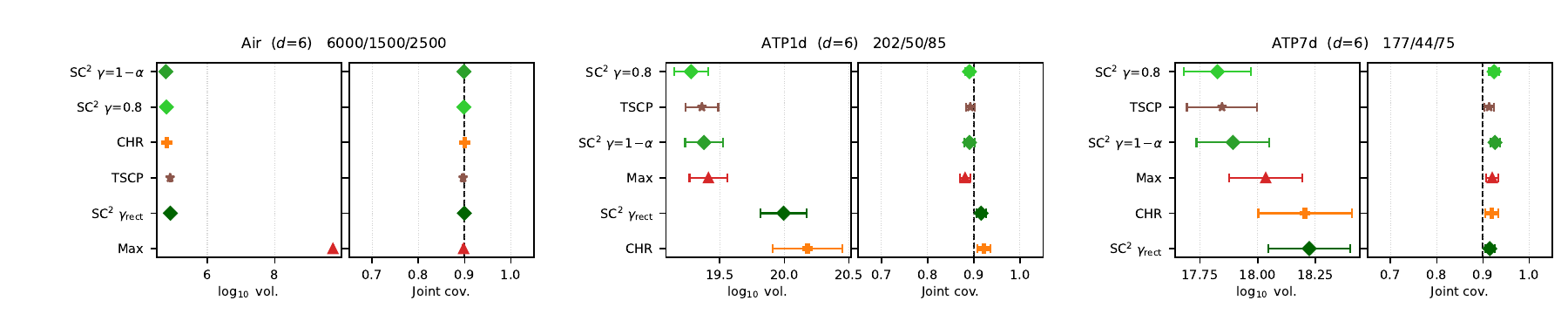}
\includegraphics[width=\textwidth]{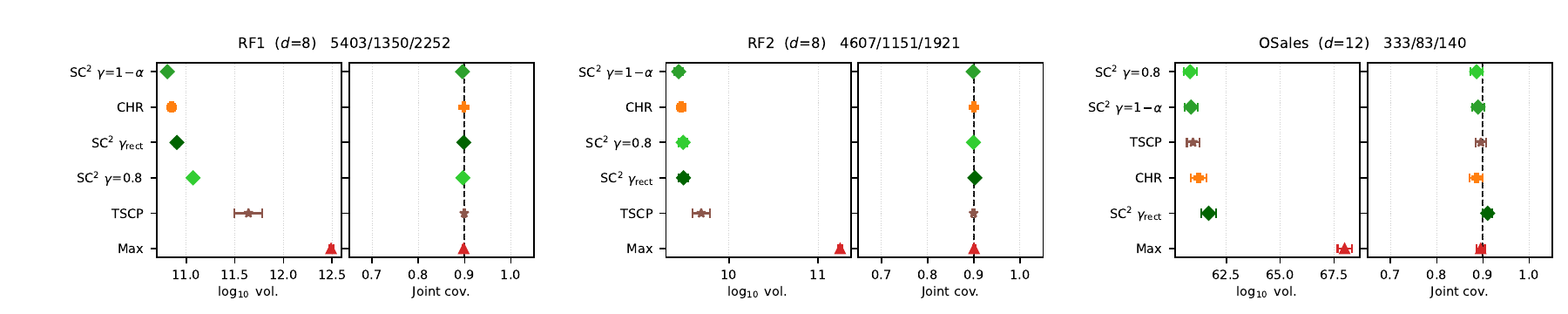}
\includegraphics[width=\textwidth]{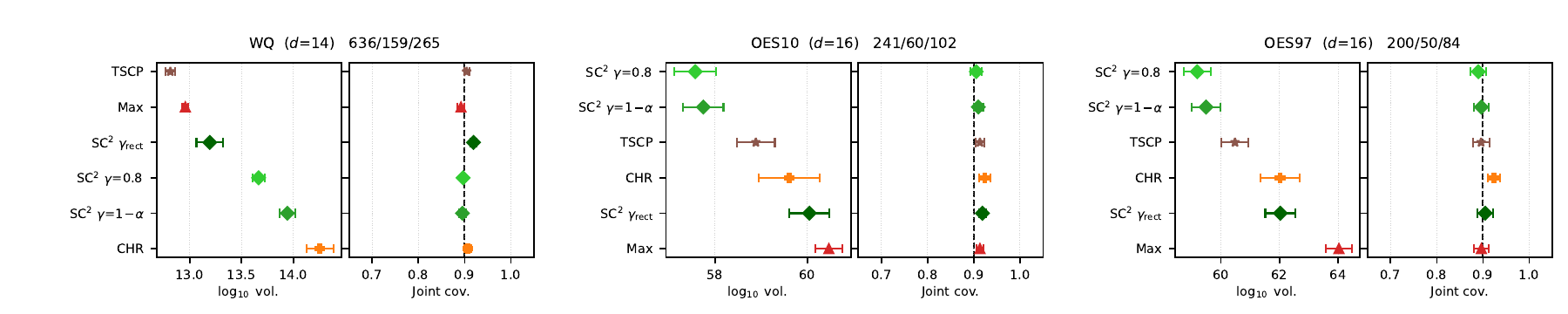}
\includegraphics[width=\textwidth]{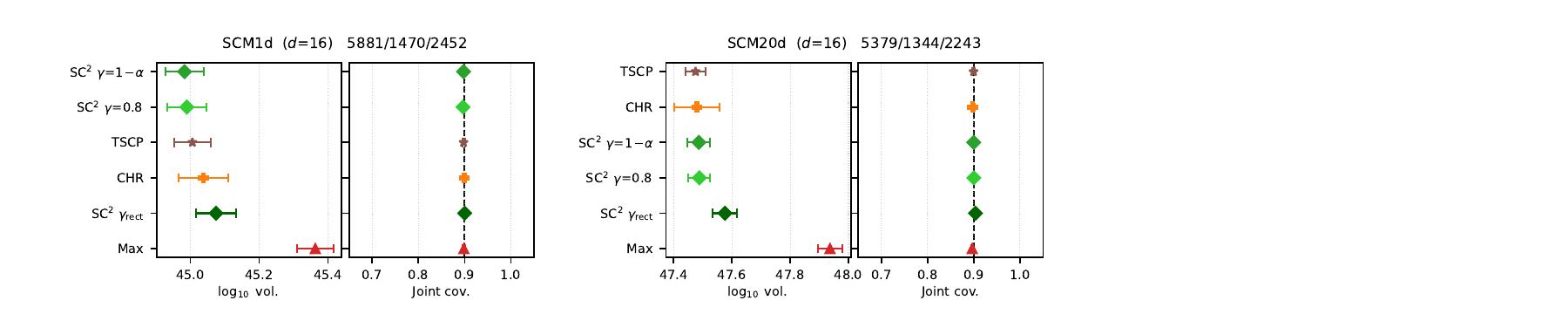}

\subsection*{Linear model, $\alpha=0.05$}
\noindent
\includegraphics[width=\textwidth]{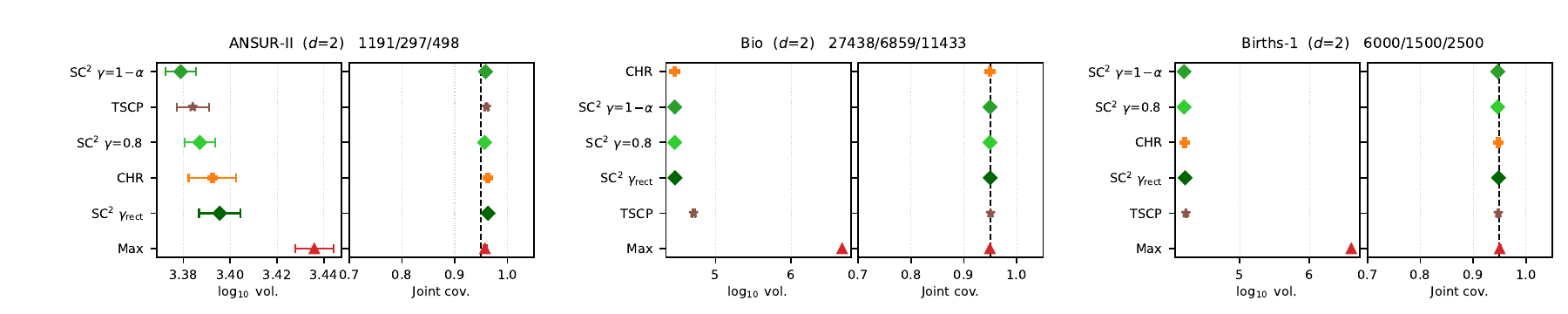}
\includegraphics[width=\textwidth]{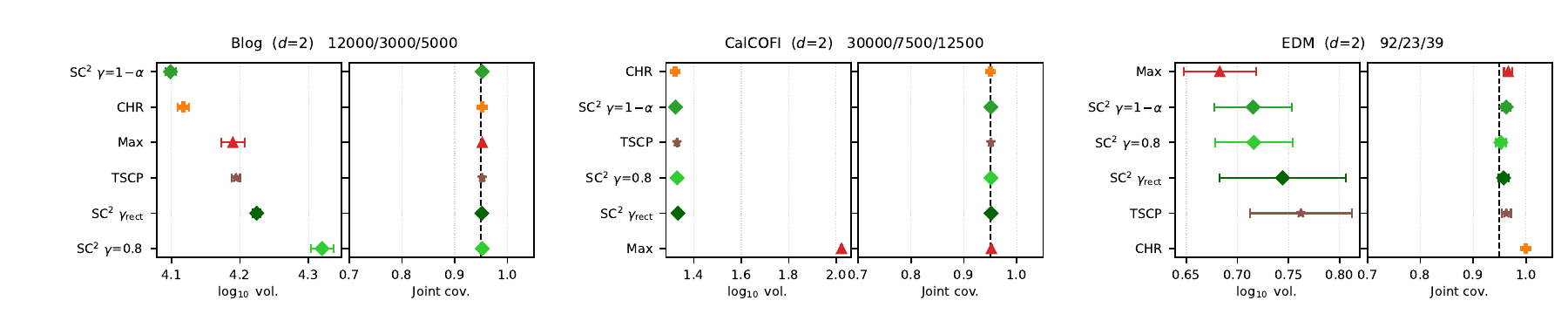}
\includegraphics[width=\textwidth]{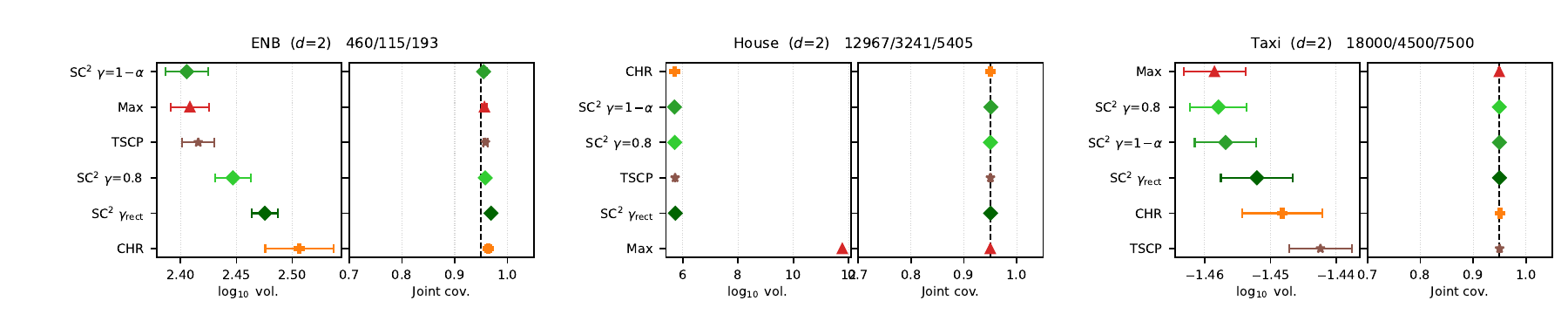}
\includegraphics[width=\textwidth]{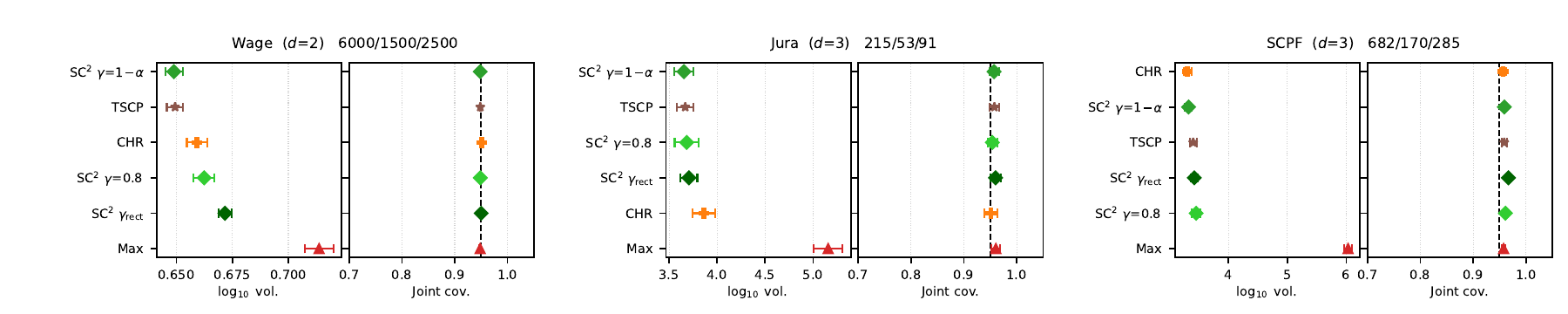}
\includegraphics[width=\textwidth]{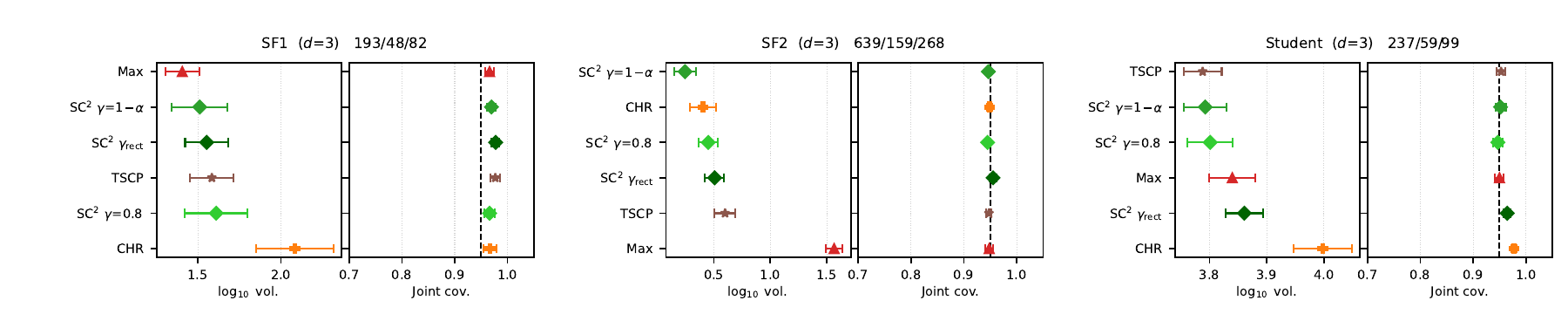}
\includegraphics[width=\textwidth]{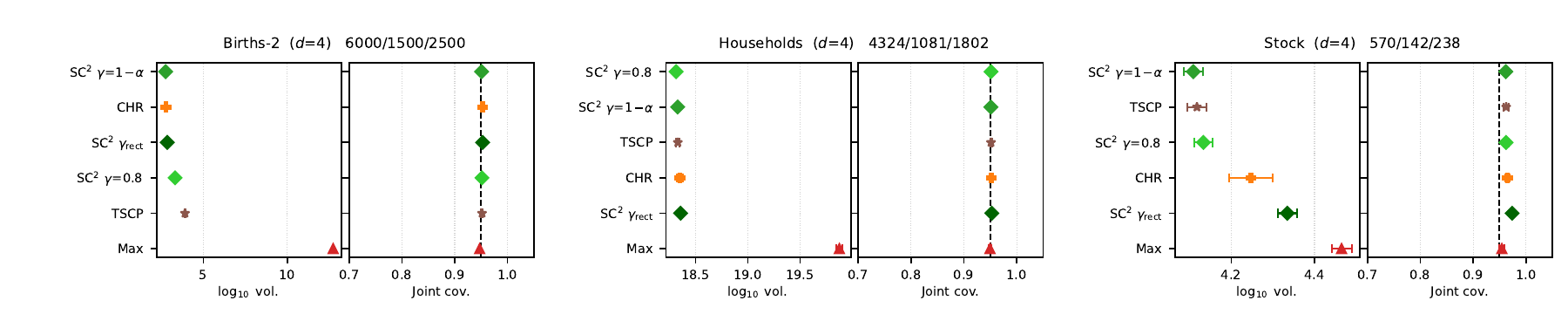}
\includegraphics[width=\textwidth]{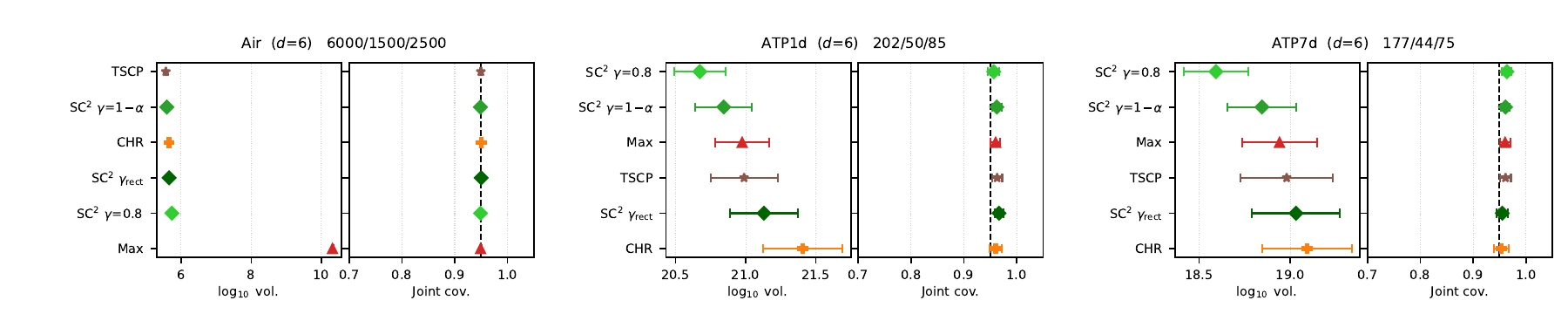}
\includegraphics[width=\textwidth]{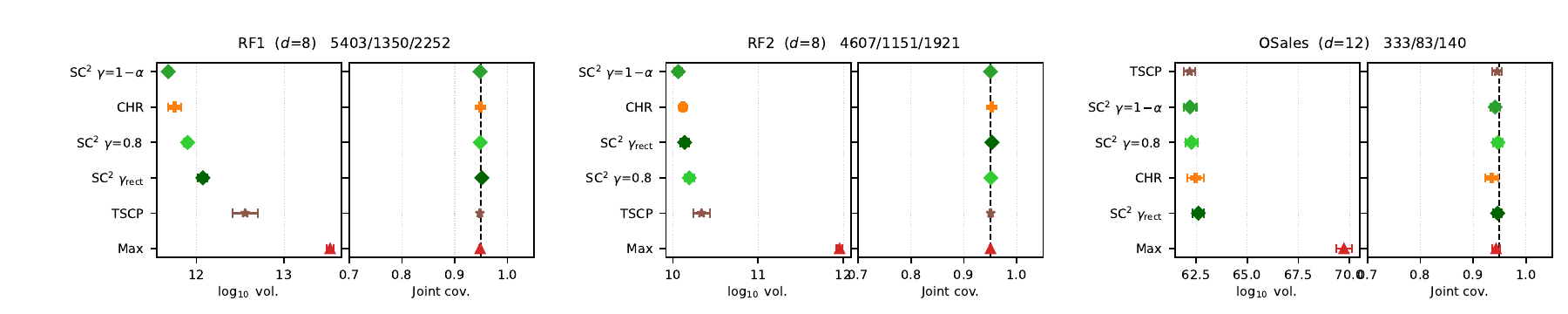}
\includegraphics[width=\textwidth]{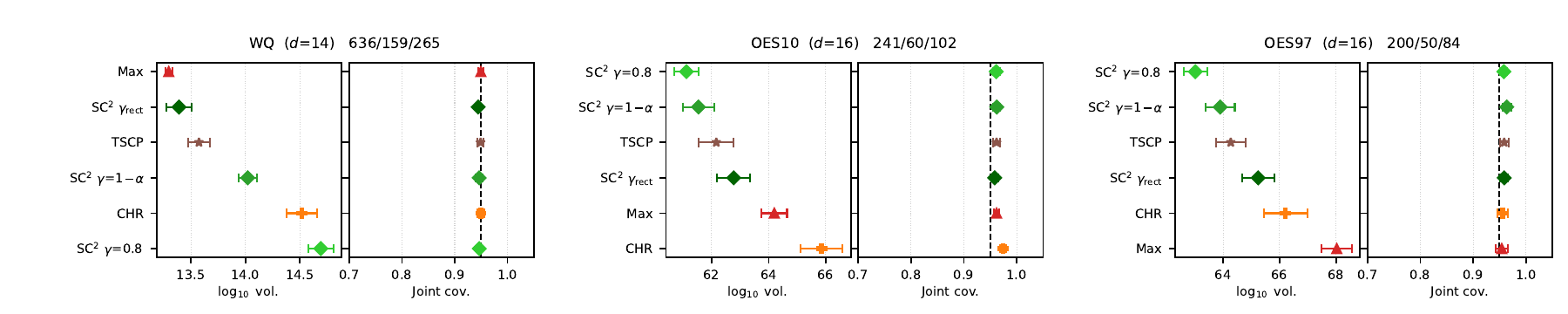}
\includegraphics[width=\textwidth]{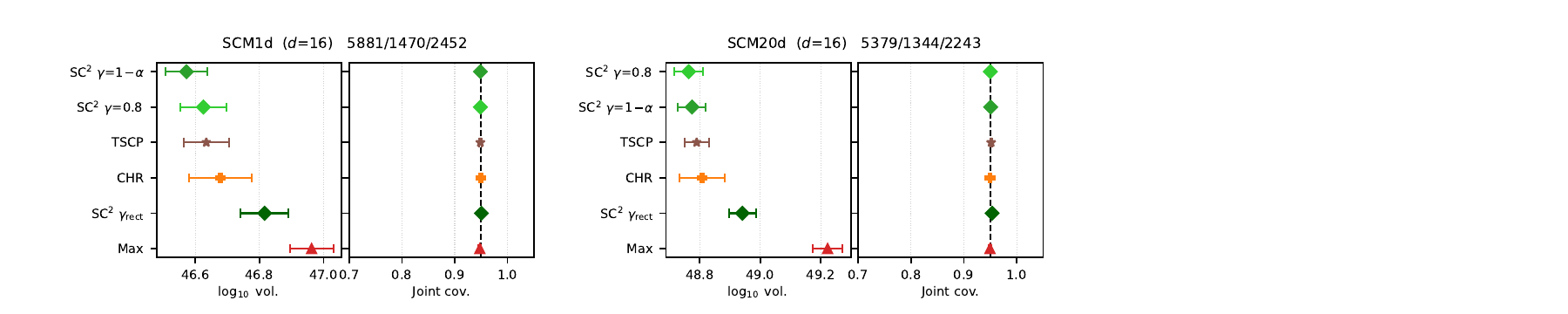}

\end{document}